\documentclass[12pt]{article}
\usepackage{graphics} 
\usepackage{epsfig} 
\usepackage{amsmath} 
\usepackage{amssymb}  
\usepackage{accents}
\usepackage{mathtools}
\usepackage{color}
\usepackage[noadjust]{cite}
\usepackage{subcaption}

\usepackage{balance}
\usepackage{verbatim}
\usepackage{amsfonts}
\usepackage{enumerate}
\usepackage{hyperref}
\usepackage{lipsum}
\usepackage{ntheorem}
\usepackage{algorithm}
\usepackage[noend]{algpseudocode}
\usepackage{authblk}

\newtheorem{theorem}{Theorem}
\newtheorem*{proof}{Proof}

\newtheorem{definition}{Definition}
\newtheorem{assumption}{Assumption}

\newtheorem{proposition}{Proposition}
\newtheorem{remark}{Remark}

\newtheorem{problem}{Problem}

\usepackage[font = footnotesize, skip=3pt]{caption}

\begin{document}

	
	\title{	Adaptive Robot Navigation with Collision Avoidance Subject to $2$nd-order Uncertain Dynamics} 
	\author{Christos K. Verginis and Dimos V. Dimarogonas}
	\affil{KTH Royal Institute of Technology, School of Electrical Engineering and Computer Science\\ \small \{cverginis,dimos\}@kth.se}
	\date{}



%
	\maketitle
	
	\begin{abstract}		  
		This paper considers the problem of robot motion planning in a workspace with obstacles for systems with uncertain $2$nd-order dynamics. In particular, we combine closed form potential-based feedback controllers with adaptive control techniques to guarantee the collision-free robot navigation to a predefined goal while compensating for the dynamic model uncertainties. We base our findings on sphere world-based configuration spaces, but extend our results to arbitrary star-shaped environments by using previous results on configuration space transformations. Moreover, we propose an algorithm for extending the control scheme to decentralized multi-robot systems. Finally, {extensive simulation results} verify the theoretical findings.		
	\end{abstract}
	

	
	\section{Introduction}
	{M}{otion} planning and specifically robotic navigation in obstacle-cluttered environments is a fundamental problem in the field of robotics \cite{lumelsky2005sensing}. Several techniques have been developed in the related literature, such as discretization of the continuous space and employment of discrete algorithms (e.g., Dijkstra, $A^\star$), probabilistic roadmaps, sampling-based motion planning, and feedback-based motion planning \cite{lavalle2006planning}. The latter, which is the focus of the current paper, offers closed-form analytic solutions by usually evaluating appropriately designed artificial potential fields, avoiding thus the potential complexity of workspace discretization and the respective algorithms. At the same time, feedback-based methods provide a solution to the control aspect of the motion planning problem, i.e., the correctness based on the solution of the closed-loop differential equation that describes the robot model.
	
	Early works on feedback-based motion planning established the Koditschek-Rimon navigation function (KRNF) \cite{koditschek1990robot,rimon1992exact}, where, through gain tuning, the robot converges safely to its goal from almost all initial conditions (in the sense of a measure-zero set).
	KRNFs were extended to more general workspaces and adaptive gain controllers  \cite{filippidis2011adjustable}, to multi-robot systems \cite{dimarogonas2006feedback,verginis2017decentralized,verginis2017robust,roussos2013decentralized}, and more recently, to convex potential and obstacles \cite{paternain2017navigation}. The idea of gain tuning has been also employed to an alternative KRNF in \cite{tanner2005towards}. 	
	Tuning-free constructions of artificial potential fields have also been developed in the related literature; 
	\cite{panagou2017distributed} tackles nonholonomic multi-robot systems, and in \cite{loizou2017navigation,vlantis2018robot} harmonic functions, also used in \cite{szulczynski2011real}, are combined with adaptive controllers to achieve almost global safe navigation. 
	A transformation of arbitrarily shaped worlds to points worlds, which facilitates the motion planning problem, is also considered in \cite{loizou2017navigation,vlantis2018robot} and in \cite{loizou2014multi} for multi-robot systems. The recent works  \cite{loizou2017navigation}, \cite{vrohidis2018prescribed} guarantee also safe navigation in  predefined \textit{time}.
	
	Barrier functions for multi-robot collision avoidance are employed in \cite{wang2017safety} and optimization-based techniques via model predictive control (MPC) can be found in {\cite{filotheou2018,mendes2017real,verginis2018communication,morgan2016swarm}}; \cite{van2011reciprocal} and {\cite{alonso2012image}} propose reciprocal collision obstacle by local decision making for the desired velocity of the robot(s). 
	Ellipsoidal obstacles are tackled in \cite{stavridis2017dynamical} and \cite{Grush18obstacle} extends a given potential field to $2$nd-order systems. A similar idea is used in \cite{montenbruck2015navigation}, where the effects of an unknown drift term in the dynamics are examined. Workspace decomposition methodologies with hybrid controllers are employed in \cite{arslan2016exact}, \cite{arslan2016coordinated}, and \cite{berkane2019hybrid}, and 
	\cite{slotine19avoidance} employs a contraction-based methodology that can also tackle the case of moving obstacles.
	
	A common assumption that most of the aforementioned works consider is the simplified robot dynamics, i.e., single integrators/unicycle kinematics, without taking into account any robot dynamic parameters.
	Hence, indirectly, the schemes depend on an embedded internal system that converts the desired velocity signal to the actual robot actuation command. The above imply that the actual robot trajectory might deviate from the desired one, jeopardizing its safety and possibly resulting in collisions.  	
	Second-order realistic robot models are considered in MPC-schemes, like \cite{mendes2017real,filotheou2018,verginis2018communication}, which might, however, result in computationally expensive solutions. Moreover, regarding model uncertainties, a global upper bound is required, which is used to enlarge the obstacle boundaries and might yield infeasible solutions. A $2$nd-order model is considered in \cite{stavridis2017dynamical}, \cite{Grush18obstacle}, without, however, considering any unknown dynamic terms. 
	The works \cite{koditschek1991control,dimarogonas2006feedback,loizou2011closed,arslan2017smooth} consider simplified $2$nd-order systems with \textit{known} dynamic terms (and in particular, inertia and gravitational terms that are assumed to be successfully compensated); \cite{montenbruck2015navigation} guarantees the asymptotic stability of $2$nd-order systems with a class of unknown drift terms to the critical points of a given potential function. However, there is no characterization of the region of attraction of the goal. {Adaptive control for constant unknown parameters is employed in \cite{cheah2009region}, where a swarm of robots is controlled to move inside a desired region.}
	
	In this paper, we consider the robot navigation in an obstacle-cluttered environment under $2$nd-order uncertain robot dynamics, in terms of {unknown mass and friction/drag terms. 
	Our main contribution lies in the design of a novel $2$nd-order smooth navigation function as well as an adaptive control law that guarantees the \textit{safe} navigation of the robot from almost all initial conditions. 	
	We also show how the proposed scheme can be applied to star-worlds, i.e., workspaces with star-shaped obstacles \cite{rimon1992exact}, as well as to decentralized multi-robot navigation. Adaptive control for multi-robot coordination was also employed in our previous works \cite{verginisLCSS,verginis2017robust}. The results in \cite{verginis2017robust}, however, are only existential, since we do not provide an explicit potential function that satisfies the desired properties, {while} \cite{verginisLCSS} focuses on the multi-agent ellipsoidal collision avoidance, without guaranteeing achievement of the primary task.}

	{The rest of the paper is organized as follows. Section \ref{sec:Notation} provides the notation used throughout the paper. Section \ref{sec:PF} describes the tackled problem and Section \ref{sec:main} provides the main results. Sections \ref{sec:Star} and \ref{sec:MAS} extend the proposed scheme to star worlds and multi-agent frameworks, respectively. Finally, simulation studies are given in Section \ref{sec:sim} and Section \ref{sec:Concl} concludes the paper.}

	\section{Notation} \label{sec:Notation}

	{The set of natural and real numbers is denoted by $\mathbb{N}$, and $\mathbb{R}$, respectively, and  $\mathbb{R}^n_{\geq 0}$, $\mathbb{R}^n_{> 0}$, $n\in\mathbb{N}$, are the $n$-dimensional sets of nonnegative and positive real numbers, respectively. The notation $\|x\|$ implies the Euclidean norm of a vector $x\in\mathbb{R}^n$. The  identity matrix is denoted by $I_n\in\mathbb{R}^{n \times n}$, the $n\times m$ matrix of zeros by $0_{n\times m}$ and the $n$-dimensional zero vector by $0_n$. The gradient and Hessian of a function $f:\mathbb{R}^n\to\mathbb{R}$ are denoted by $\nabla_x f(x) \coloneqq \frac{\partial f(x)}{\partial x}\in\mathbb{R}^n$ and $\nabla_x^2 f(x) \in\mathbb{R}^{n\times n}$, respectively.} 
%
%
%
%
%

	\section{Problem Statement}  \label{sec:PF}

	Consider a spherical robot operating in a bounded workspace $\mathcal{W}$, characterized by its {position vector} $x \in \mathbb{R}^n$, $n\in\{2,3\}$ and radius $r > 0$, and subject to the dynamics:

	\begin{subequations} \label{eq:dynamics}	
	\begin{align}
		& \dot{x}= v \\
		& m \dot{v} + f(x,v) + mg = u,
	\end{align}
	\end{subequations}
	where $m > 0$ is the \textit{unknown} mass, $g \in \mathbb{R}^n$ is the constant gravity vector, $u\in\mathbb{R}^n$ is the input vector, and $f:\mathbb{R}^{2n}\to \mathbb{R}^n$ is an {\textit{unknown}} friction-like function, satisfying the following assumption:
	\begin{assumption} \label{ass:f}
		The function $f:\mathbb{R}^{2n}\to \mathbb{R}^n$ is analytic and satisfies 	
	\begin{equation}
		\|f(x,v)\| \leq \alpha \|v\|,
	\end{equation} 
	$\forall x,v \in \mathbb{R}^{2n}$, where $\alpha \in \mathbb{R}_{\geq 0}$ is an unknown constant.
	\end{assumption}	
	{The aforementioned assumption is inspired by standard friction-like terms, which can be approximated by continuously differentiable velocity functions \cite{makkar2005new}. Constant unknown friction terms could be also included in the dynamics (e.g., incorporated in the constant gravity vector).}
	Note also that $\|f(x,v)\| \leq \alpha \|v\|$ implies $f(x,0_n) = 0_n$, and $\frac{\partial f(x,v)}{\partial x} \Big|_{v=0_n} = 0_{n\times n}$.
    The workspace is assumed to be an open ball centered at the origin 

	\begin{equation} \label{eq:workspace}
		\mathcal{W} \coloneqq \{q \in \mathbb{R}^n : \|q\| < r_{\mathcal{W}} \},
	\end{equation}
	where $r_{\mathcal{W}} > 0$ is the workspace radius. The workspace contains $M\in\mathbb{N}$ closed sets $\mathcal{O}_j$, ${j\in\mathcal{J}} \coloneqq \{1,\dots,M\}$, corresponding to obstacles. Each obstacle is a closed ball centered at $c_j\in\mathbb{R}^3$, with radius ${r_{o_j}} > 0$, i.e., $\mathcal{O}_j \coloneqq \{q\in\mathcal{W} : \|q-{c_j}\|\leq {r_{o_j}} \}, \ \ \forall j\in\mathcal{J}$.
	The analysis that follows will be based on the transformed workspace: 
	\begin{equation} \label{eq:transf. workspace}
		\bar{\mathcal{W}} \coloneqq \{q \in \mathbb{R}^n : \|q\| < \bar{r}_\mathcal{W} \coloneqq r_\mathcal{W} - r\},
	\end{equation}  
	and set of obstacles $\bar{\mathcal{O}}_j \coloneqq \{q\in\mathcal{W} : \|q-{c_j}\|\leq {\bar{r}_{o_j}}\coloneqq {r_{o_j}} + r \}, \ \ \forall j\in\mathcal{J},$ 
	where the robot is reduced to the point $x$.
	The free space is defined as 
	\begin{equation} \label{eq:sphere world}
		\mathcal{F} \coloneqq \bar{\mathcal{W}} \backslash \bigcup_{j\in\mathcal{J}} \bar{\mathcal{O}}_j,
	\end{equation}
	also known as a \textit{sphere world} {\cite{koditschek1990robot}}.
	We consider the following {common feasibility assumption \cite{koditschek1990robot,vrohidis2018prescribed}} for $\mathcal{F}$:
	\begin{assumption} \label{ass: workspace}
		The workspace $\mathcal{W}$ {and} the obstacles $\mathcal{O}_j$ satisfy $\|c_i - c_j\| > r_{o_i} + r_{o_j} + 2r$ and $r_\mathcal{W} - \|c_j\| > r_{o_j} + 2r$, $\forall i,j\in\mathcal{J},i\neq j$.
	\end{assumption}
   Assumption \ref{ass: workspace} implies that we can find some $\bar{r} > 0$ such that 
	\begin{subequations} \label{eq:r_bar}
	\begin{align}
		& \|c_i - c_j\| > r_{o_i} + r_{o_j} + 2r + 2\bar{r}, \ \ \forall i,j \in \mathcal{J}, i\neq j, \\
		& r_\mathcal{W} - \|c_j\| > r_{o_j} + 2r + 2\bar{r}, \ \ \forall j\in \mathcal{J}
	\end{align}
	\end{subequations}


	This paper treats the problem of navigating the robot to a destination $x_{\textup{d}}$ while avoiding the obstacles and the workspace boundary, formally stated as follows:
	\begin{problem} \label{prob 1}
		Consider a robot subject to the \textit{uncertain} dynamics \eqref{eq:dynamics}, operating in the aforementioned sphere world, with $(x(t_0), v(t_0)) \in \mathcal{F}\times\mathbb{R}^n$. Given a destination $x_{\textup{d}}\in \mathcal{F}$, design a control protocol $u$ such that 
		\begin{align*}
			& x(t) \in \mathcal{F}, \ \ t \geq t_0 \\
			& \lim_{t\to\infty}(x(t),v(t))  = (x_{\textup{d}},0_n)
		\end{align*} 
	\end{problem}
	\section{Main Results} \label{sec:main}

	We provide in this section our methodology for solving Problem \ref{prob 1}. Define first  the set $\bar{\mathcal{J}} \coloneqq \{0\}\cup\mathcal{J}$ as well as the distances $d_j: \mathcal{F} \to \mathbb{R}_{\geq 0}$, $j\in\bar{\mathcal{J}}$, with $d_j(x) \coloneqq \|x-c_j\|^2 - \bar{r}_{o_j}^2$, $\forall j\in\mathcal{J}$, and $d_0(x) \coloneqq \bar{r}_{\mathcal{W}}^2 - \|x\|^2$. Note that, by keeping $d_j(x) > 0$, $d_0(x) > 0$, we guarantee that $x \in \mathcal{F}$\footnote{A safety margin can also be included, which needs, however, to be incorporated in the constant $\bar{r}$ of \eqref{eq:r_bar}.}.
	{We {also define} the constant 
	\begin{equation} \label{eq:bar_r_d}
	\bar{r}_{\textup{d}} \coloneqq \min \left\{\bar{r}_{\mathcal{W}}^2-\|x_{\textup{d}}\|^2, \min_{j\in\mathcal{J}}\left\{ d_j(x_\textup{d}) \right\} \right\}
	\end{equation}
	as the minimum distance of the goal to the obstacles/workspace boundary.}
	We introduce next the notion of the \textit{$2$nd-order navigation function}: 
	\begin{definition} \label{def:2nd nf}
		A \textit{$2$nd-order navigation function} is a function $\phi: \mathcal{F} \to \mathbb{R}_{\geq 0}$ of the form 

		\begin{equation*} 
			\phi(x) \coloneqq k_1\|x - x_{\textup{d}}\|^2 + k_2\sum_{j\in\bar{\mathcal{J}}}\beta(d_j(x)),
		\end{equation*}		
		where $\beta:\mathbb{R}_{>0}\to\mathbb{R}_{\geq 0}$ is a (at least) twice contin. differentiable function and $k_1, k_2$ are positive constants, with the followings properties:
		\begin{enumerate}
			\item $\beta((0,\tau])$ is strictly decreasing,   				
			$\lim_{{z} \to 0} {\beta}({z}) = \infty$, and $\beta({z}) = \beta(\tau)$, $\forall {z} \geq \tau$, $j\in\bar{\mathcal{J}}$, for some $\tau > 0$, 
			\item $\phi(x)$ has a global minimum  at $x = x_{\textup{d}} \in \textup{int}(\mathcal{F})$ where $\phi(x_{\textup{d}}) = 0$,			
			\item if $\beta'(d_k(x)) \neq 0$ {and $\beta''(d_k(x)) \neq 0$} for some $k\in\bar{\mathcal{J}}$, then $\beta'(d_j(x)) = \beta''(d_j(x)) = 0$, for all $j\in\bar{\mathcal{J}} \backslash\{k\}$.

			\item The function $\widetilde{\beta}:(0,\tau) \to \mathbb{R}_{\geq 0}$, with
				$\widetilde{\beta}({z}) \coloneqq \beta''({z}) {z}\sqrt{{z}}$ 
			is strictly decreasing.
		\end{enumerate}	
	\end{definition}
	
	By using the first property we will guarantee that, by keeping $\beta(d_j(x))$ bounded, there are no collisions with the obstacles or the free space boundary. Property $2$ will be used for the asymptotic stability of the desired point $x=x_{\textup{d}}$. Property $3$ places the rest of the critical points of $\phi$  (which are proven to be saddle points) close to the obstacles, and the last property is used to guarantee that these are non-degenerate.	
	{An example for $\beta$ that satisfies properties 1) and 4), is
	\begin{align}
	\beta(z) \coloneqq& 
	\begin{cases} \displaystyle
	(6z^5 - 15z^4 + 10z^3)^{-1}, &  z \leq 1\\
	1, & z \geq 1 ,
	\end{cases} \label{eq:beta_exps} 	 
	\end{align}}
	Note that $\beta$ is essentially a reciprocal barrier function 		\cite{wang2017safety}.
	We prove next that, by appropriately choosing $\tau$, only one ${\beta(d_j(x))}$, $j\in\bar{\mathcal{J}}$ affects the robotic agent for each $x\in\mathcal{F}$, and furthermore that {${\beta'(d_j(x_\textup{d}))} = {\beta''(d_j(x_\textup{d}))} = 0$.  Hence, properties 2) and 3) of Def. \ref{def:2nd nf} are satisfied}.
	\begin{proposition} \label{prop:tau}
	By choosing $\tau$ as $\tau \in (0,\min\{\bar{r}^2, \bar{r}_{\textup{d}}\})$,
	where $\bar{r}, \bar{r}_{\textup{d}}$ were introduced in \eqref{eq:r_bar} and \eqref{eq:bar_r_d}, respectively,  we guarantee that at each $x\in \mathcal{F}$	 
	{there is no more than one $j\in\bar{\mathcal{J}}$} such that $d_j \leq \tau$, implying that $\beta'(d_j(x))$ and $\beta''(d_j(x))$ are non-zero. 
	\end{proposition}	

	\begin{proof}
	{Assume that $d_j(x) \leq \tau$ for some $j\in\mathcal{J}$, $x\in\mathcal{F}$. Then, in view of \eqref{eq:r_bar}, it holds that 
	\begin{align*}
		 \|x - c_j\|^2  < \bar{r}^2 + \bar{r}^2_{o_j} \Rightarrow  
		 \|x - c_j\| < \bar{r} + \bar{r}_{o_j} =  \bar{r} + r + r_{o_j} < \|c_j - c_k\|    
	\end{align*}
	$\forall k\in\mathcal{J}\backslash\{j\}$, and hence
	\begin{align*}
		\|x - c_k\| &= \|x - c_j + c_j - c_k\|  
		 \geq \|c_j - c_k\| - \|x - c_j\|>  r_{o_k} + r + \bar{r} \Rightarrow \\
		\|x - c_k\|^2 &> (r_{o_k} + r + \bar{r})^2 > (r_{o_k} + r)^2 + \bar{r}^2,
	\end{align*}
	implying $ d_k(x) > \bar{r}^2 > \tau$, $\forall k\in\mathcal{J}\backslash\{j\}$. Moreover, in view of \eqref{eq:r_bar}, it holds that 
	\begin{align*}
		\|x\| &\leq \|x-c_j\| + \|c_j\| \pm r_\mathcal{W} \Rightarrow \\
		\|x\| &< r_\mathcal{W} - r -\bar{r} \Rightarrow  (r_{\mathcal{W}}-r)^2  \geq (\|x\| + \bar{r})^2 \Rightarrow \\
       \bar{r}_\mathcal{W}^2 &\geq \|x\|^2 + \bar{r}^2  \Rightarrow  \bar{r}^2_\mathcal{W} - \|x\|^2 > \bar{r}^2, 
	\end{align*}
	and hence 
	$d_o(x) > \tau$. Similarly, we conclude by contradiction that $d_o(x) \leq \tau \Rightarrow d_j > \tau$, $\forall j\in\mathcal{J}$. }
	\end{proof}

	{Moreover, it holds for the desired equilibrium that 
	\begin{align*}
		&x = x_\textup{d} \Leftrightarrow d_j(x) = \|x_\textup{d} - c_j\|^2 - \bar{r}_{j}^2 \geq \bar{r}_\textup{d} > \tau, 
	\end{align*}	
	and 
	\begin{align*}
		& x = x_\textup{d} \Leftrightarrow d_0(x) = \bar{r}_\mathcal{W}^2 - \|x_\textup{d}\|^2 \geq \bar{r}_\textup{d} > \tau,
	\end{align*}
	and hence $\beta'(d_j(x_\textup{d})) = \beta''(d_j(x_\textup{d})) = 0$, $\forall j\in\bar{\mathcal{J}}$.}

	 Intuitively, the obstacles and the workspace boundary  have a local region of influence defined by the constant $\tau$, which will play a significant role in determining the stability of the overall scheme later. {This robot interaction with only one obstacle at a time has also been demonstrated in the feedback control-based related literature, e.g., 
	 \cite{arslan2016exact,vrohidis2018prescribed,filippidis2011adjustable,lionis2007locally,paternain2017navigation}, which {deals} with simplified single-integrator  models, as well as in the more discrete decision making \textit{bug} algorithms {\cite{lumelsky2005sensing}}, which involve circumnavigation of obstacles and can handle in general complex unknown environments.}
	 

The expressions for the gradient and the Hessian of $\phi$, which will be needed later, are the following:

\small
	\begin{subequations} \label{eq:grad + hessian}
		\begin{align}
		\nabla_x \phi(x) =& 2k_1(x - x_{\textup{d}}) + 2k_2\sum_{j\in\mathcal{J}} \beta'(d_j)(x - c_j) - 2k_2\beta'(d_0)x \\
		\nabla_x^2 \phi(x) =& 2 \left(k_1 - k_2\beta'(d_0) + k_2\sum_{j\in\mathcal{J}} \beta'(d_j)\right) I_n -  2k_2 \beta''(d_0)xx^\top +  \notag \\&
		2k_2\sum_{j\in\mathcal{J}}\beta''(d_j)(x - c_j)(x - c_j)^\top.
		\end{align}
	\end{subequations}
	\normalsize
	Given the aforementioned definitions, we design a reference signal $v_{\textup{d}}:\mathcal{F}\to\mathbb{R}^n$ for the robot velocity $v$ as 
	\begin{equation} \label{eq:v_d}
		v_{\textup{d}}(x) = -\nabla_x \phi(x).
	\end{equation}
	Next, we design the control input $u$ to guarantee tracking of the aforementioned reference velocity  as well as compensation of the unknown terms $m$ and $f(x,v)$. More specifically, we define the signals $\hat{m}\in\mathbb{R}$ and $\hat{\alpha}\in\mathbb{R}$ as the estimation terms of $m$ and $\alpha$ (see Assumption \ref{ass:f}), respectively, and the respective errors $\tilde{m} \coloneqq \hat{m} - m$, $\widetilde{\alpha} \coloneqq \hat{\alpha} - \alpha$. We design now the control law $u:\mathcal{F}\times\mathbb{R}^{n+2}\to\mathbb{R}^n$ as  $u\coloneqq u(x,v,\hat{m},\hat{\alpha})$, with

	\begin{align} \label{eq:u}
		 u&\coloneqq  -k_\phi \nabla_x \phi(x) + \hat{m}(\dot{v}_{\textup{d}} + g) - \left(k_v + {\frac{3}{2}}\hat{\alpha}\right)e_v,
	\end{align}	
	where $e_v \coloneqq v - v_{\textup{d}}$, and $k_v$, $k_\phi$ are positive gain constants. Moreover, we design the adaptation laws for the estimation signals as 
	\begin{subequations} \label{eq:adaptation laws}
	\begin{align}
		\dot{\hat{m}} \coloneqq& -k_m e_v^\top( \dot{v}_{\textup{d}}+g)  \\
		\dot{\hat{\alpha}} \coloneqq & k_\alpha \|e_v\|^2, 
	\end{align}
	\end{subequations}
	with $k_m$, $k_\alpha$ positive gain constants, $\hat{\alpha}(t_0) \geq 0$, and arbitrary finite initial condition $\hat{m}(t_0)$. {As will be verified by the proof of Theorem \ref{th:single robot}, the choices for the control and adaptation laws are based on {Lyapunov techniques, and follow standard adaptive} control methodologies (see, e.g., \cite{lavretsky2013robust}).	}	
	\begin{theorem}\label{th:single robot}
		Consider a robot operating in $\mathcal{W}$, subject to the uncertain $2$nd-order dynamics \eqref{eq:dynamics}. Given $x_{\textup{d}} \in \mathcal{F}$, the control protocol \eqref{eq:v_d}-\eqref{eq:adaptation laws}  guarantees the collision-free navigation to $x_{\textup{d}}$ from almost all initial conditions $(x(t_0),v(t_0),\hat{m}(t_0),\hat{\alpha}(t_0)) \in \mathcal{F}\times\mathbb{R}^{{n+1}}\times\mathbb{R}_{\geq 0}$, given a sufficiently small $\tau$ and that $k_\phi > \frac{\alpha}{2}$. Moreover, all closed loop signals remain bounded, $\forall t \geq t_0$.
	\end{theorem}  

	\begin{proof}
		Consider the Lyapunov candidate function 
		\begin{equation} \label{eq:V Lyap}
			V \coloneqq k_{\phi}\phi + \frac{m}{2}\|e_v\|^2 + {\frac{3}{4k_\alpha}}\widetilde{\alpha}^2 + \frac{1}{2k_m}\widetilde{m}^2.
		\end{equation}		
		Since $x(t_0) \in \mathcal{F}$, there exists {a constant} $\bar{d}_j$ such that $d_j(x(t_0)) \geq \bar{d}_j > 0$, $j\in\bar{\mathcal{J}}$, and hence {there exists} a finite positive constant $\bar{V}_0$ such that $V(t_0) \leq \bar{V}_0$.
		By considering the time derivative of $V$ and using $v = e_v + v_{\textup{d}}$ and Assumption \ref{ass:f}, we obtain after substituting \eqref{eq:adaptation laws}:
		\begin{align*}
			\dot{V} =& k_{\phi} \nabla_x \phi(x)^\top (e_v + v_{\textup{d}}) + e_v^\top(u - mg - f(x,v) - m\dot{v}_{\textup{d}} ) \\
			&+{\frac{3}{2}}\widetilde{\alpha}\|e_v\|^2 - \widetilde{m}e_v^\top (\dot{v}_{\textup{d}} + g)\\
			\leq & -k_\phi \|\nabla_x \phi(x) \|^2 + e_v^\top(k_{\phi}\nabla_x \phi(x) + u - m(g +\dot{v}_{\textup{d}})) + \\
			& \alpha\|e_v\|\|v\| +  {\frac{3}{2}}\widetilde{\alpha}\|e_v\|^2 - \widetilde{m}e_v^\top (\dot{v}_{\textup{d}} + g), 		
		\end{align*}
		which, by substituting \eqref{eq:u} and using $\alpha \|e_v\|\|v\| \leq \frac{\alpha}{2}\|e_v\|^2 + \frac{\alpha}{2}\|\nabla_x \phi(x)\|^2 + \frac{\alpha}{2}\|e_v\|^2$, becomes
		\begin{align*}
			\dot{V} \leq & -\left(k_\phi - \frac{\alpha}{2}\right) \|\nabla_x \phi(x) \|^2  - {\frac{3}{2}}\hat{\alpha}\|e_v\|^2  + 
			{\frac{3}{2}}\alpha\|e_v\|^2  \\
			& - k_v\|e_v\|^2+ \widetilde{m}e_v^\top (g+\dot{v}_{\textup{d}}) + \widetilde{\alpha}\|e_v\|^2 - \widetilde{m}e_v^\top (\dot{v}_{\textup{d}}+g) \\
			= & 
			-\left(k_\phi - \frac{\alpha}{2}\right) \|\nabla_x \phi(x) \|^2 - k_v\|e_v\|^2 \leq 0.
		\end{align*}
		Hence, we conclude that $V(t)$ is non-increasing, and hence $\beta(d_j(x(t))) \leq V(t) \leq V(t_0) \leq \bar{V}_0$, $\forall t\geq t_0$, which implies that collisions with the obstacles and the workspace boundary are avoided, i.e., $$x(t) \in \bar{\mathcal{F}} \coloneqq \left\{ x\in \mathcal{F}: \beta(d_j(x)) \leq \bar{V}_0, \forall j\in\bar{\mathcal{J}}\right\},$$ $\forall t \geq t_0$. Moreover,
		\eqref{eq:grad + hessian} implies also the boundedness of $\nabla_x \phi(x)|_{x(t)}$, $\forall t\geq t_0$.
		In addition, the boundedness of $V(t)$ implies also the boundedness of $x(t)$, $e_v(t)$, $\widetilde{m}(t)$, $\widetilde{\alpha}(t)$, $\widetilde{g}(t)$ and hence of $v(t)$,  $\hat{m}(t)$, $\hat{\alpha}(t)$, $\forall t\geq t_0$. More specifically, 
		by letting  $s \coloneqq [x^\top, v^\top, \widetilde{\alpha}$, $\widetilde{m}]^\top$, we conclude that $s(t)\in \bar{S}$, $\forall t\geq t_0$, with

		\begin{align*}
			\bar{S} \coloneqq& \{ s\in\bar{\mathcal{F}}\times \mathbb{R}^{n+2} : |\widetilde{\alpha}| \leq \sqrt{{\tfrac{4}{3}}k_\alpha\bar{V}_0}, 
			|\widetilde{m}| \leq \sqrt{2k_m\bar{V}_0}, \notag \\ 
			& \|v\| \leq \sqrt{2m\bar{V}_0} + \sup_{x\in\bar{\mathcal{F}}}\|\nabla_x\phi(x)\| \}
		\end{align*}
		 Therefore, by invoking LaSalle's invariance principle, we conclude that the solution $s(t)$ will converge to the largest invariant set in $ S \coloneqq \{s\in\bar{S}:  \dot{V} = 0\}$, which, in view of \eqref{eq:v_d}, becomes $S \coloneqq \{s\in\bar{S} : {\nabla_x \phi(x) = 0, v = 0}\}$. Consider now the {closed-loop} dynamics for $s$:
		 \begin{subequations} \label{eq:cl system}
		\begin{align}
			\dot{x} =& v \\
			\dot{v} =& \frac{1}{m}(\widetilde{m}g + \hat{m}\dot{v}_{\textup{d}}- k_{\phi}\nabla_x \phi(x)- f(x,v)    - \left(k_v + {\frac{3}{2}}\hat{\alpha}\right)(v + \nabla_x\phi(x)))  \label{eq:cl system v_dot} \\ 
			\dot{\widetilde{m}} =& -k_m (v + \nabla_x\phi(x))^\top (\dot{v}_{\textup{d}}+ g) \\ 			 
			\dot{\widetilde{\alpha}} =& k_\alpha \|v + \nabla_x\phi(x)\|^2.
		\end{align}
		\end{subequations}
		Note that, in view of the aforementioned discussion and the continuous differentiability of $f(x,v)$, the right-hand side of \eqref{eq:cl system v_dot} is bounded in $\bar{S}$. Note also that \eqref{eq:grad + hessian} implies the boundedness of $\nabla^2_x\phi(x)$ in $\bar{\mathcal{F}}$. Moreover, by differentiating $\dot{v}$, using the closed loop dynamics \eqref{eq:cl system} and \eqref{eq:grad + hessian}, we conclude the boundedness of $\ddot{v}$ and the uniform continuity of $\dot{v}(t)$ in $\bar{S}$. Hence, since $\lim_{t\to\infty}v(t) = 0$, we invoke Barbalat's Lemma to conclude $\lim_{t\to\infty}\dot{v}(t) = 0$.
		
		Therefore, the set $S$ consists of the points where $\dot{v} = v= \nabla_x \phi(x) = 0$, $\dot{v}_{\textup{d}} = \nabla_x^2\phi(x) v = 0$, and by also using the property $f(x,0) = 0$ we obtain $\lim_{t\to\infty}\widetilde{m}(t) = 0$ and $\lim_{t\to\infty}\dot{s}(t) = 0$.
		Note also that $\hat{\alpha}:[t_0,\infty)\to\mathbb{R}_{\geq 0}$ is a monotonically increasing function and it converges thus to some constant positive value $\hat{\alpha}^\star > 0$, since {$\hat{\alpha}(t_0) \geq 0$ and $\lim_{t\to\infty}\dot{\hat{\alpha}}(t) = \lim_{t\to\infty}\dot{\widetilde{\alpha}}(t) = 0$}.
		Therefore, we conclude that the system will converge to an equilibrium $s^\star \coloneqq [(x^\star)^\top, 0_n^\top, 0, \hat{\alpha}^\star]$ satisfying $\nabla_x \phi(x)|_{x^\star} = 0$.		
		Since $\lim_{t\to\infty} \nabla_x \phi(x)|_{x(t)} = \lim_{t\to\infty} v(t) = 0$, the system converges to the critical points of $\phi(x)$, i.e., we obtain from \eqref{eq:grad + hessian} that at steady state: 
		\begin{equation} \label{eq:grad at equil}
			2k_1 (x^\star - x_{\textup{d}}) = -k_2 \sum_{j\in\bar{J}} \beta'(d_j^\star)(x^\star-c_j),
		\end{equation}
		where $d_j^\star \coloneqq d_j(x^\star)$, $\forall j\in\bar{\mathcal{J}}$. 
		According to the choice of $\tau$ in Prop. \ref{prop:tau}, $x^\star=x_{\textup{d}}$ implies that $\beta'(d_j^\star) = 0$, $\forall j\in\bar{\mathcal{J}}$, and hence the desired equilibrium $x^\star=x_{\textup{d}}$ satisfies \eqref{eq:grad at equil}. Other \textit{undesired} critical points of $\phi(x)$ consist of cases where the two sides of \eqref{eq:grad at equil} cancel each other out. However, as already proved, only one $\beta'_j$ can be nonzero for each $x\in\mathcal{F}$. Hence, the undesired critical points satisfy one of the following expressions:  
		\begin{subequations} \label{eq:grad at equil k}	
		\begin{align} 
			k_1 (x^\star - x_{\textup{d}}) =& -k_2 \beta'(d_k^\star)(x^\star-c_k), \label{eq:grad at equil k in J}\\
			k_1 (x^\star - x_{\textup{d}}) =& k_2 \beta'(d_0^\star)x^\star, \label{eq:grad at equil k=0}
		\end{align}  
		\end{subequations}
		for some $k\in\bar{\mathcal{J}}$. In the case of \eqref{eq:grad at equil k=0}, $x^\star$ is collinear with the origin and $x_{\textup{d}}$. However, 	
		the choice of $\tau < \bar{r}_\mathcal{W}^2 - \|x_{\textup{d}}\|^2$ in Prop. \ref{prop:tau} implies that $$d_0^\star = \bar{r}^2_\mathcal{W} - \|x^\star\|^2 \leq \tau < \bar{r}^2_\mathcal{W} - \|x_\textup{d}\|^2 \Leftrightarrow \|x^\star\| \geq \|x_\textup{d}\|,$$
		 and hence $x^\star - x_{\textup{d}}$ and $x^\star$ have the same direction. Therefore, since $\beta'(d_j) < 0$, for $d_j < \tau$, $\forall j\in\bar{\mathcal{J}}$, \eqref{eq:grad at equil k=0} is not feasible.
		
		Moreover, in the case of \eqref{eq:grad at equil k in J}, since $\beta' \leq 0$, $x^\star - x_{\textup{d}}$ and $x^\star - c_k$ point to the same direction. Hence, the respective critical points $x^\star$ are on the $1$D line connecting $x_{\textup{d}}$ and $c_k$. Moreover, since $\tau < \bar{r}_{\textup{d}} \leq \|x_\textup{d} - c_k\|^2 - \bar{r}_{o_k}^2$, as chosen in Prop. \ref{prop:tau}, it holds that 
		\begin{align*}
			& d_k^\star = \|x^\star - c_k\|^2 - \bar{r}_{o_k}^2 < \|x_\textup{d} - c_k\|^2 - \bar{r}_{o_k}^2 \Leftrightarrow \\
			& \|x^\star - x_{\textup{d}}\| > \|x^\star - c_k\|.
 		\end{align*}		
		We proceed now by showing that the critical points satisfying \eqref{eq:grad at equil k in J} are saddle points, which have a lower dimension stable manifold. 
		Consider, therefore, the error $e_x = x - x^\star$, where $x^\star \neq x_{\textup{d}}$ represents the potential \textit{undesired} equilibrium point that satisfies \eqref{eq:grad at equil k in J}. {Let also $s_e \coloneqq [s_x^\top,\widetilde{\alpha}^\top]^\top$, where $s_x \coloneqq [e_x^\top, v^\top, \widetilde{m}]^\top$}, whose linearization around zero yields, after using \eqref{eq:cl system} and $\frac{\partial f(x,v)}{\partial x}\Big|_{v = 0_n} = 0_{n\times n}$, 
		\begin{equation} \label{eq:s_e dot}
			\dot{s}_e = \bar{A}_s s_e,
		\end{equation}	
		where
		\begin{align} 
			\bar{A}_s \coloneqq \begin{bmatrix}
			A_s & 0_{2n+2} \\ 
			0_{2n+2}^\top & 0
			\end{bmatrix}, 
			A_s \coloneqq \begin{bmatrix}
			0_{n\times n} & I_n & 0_n \\
			A_{s,21} & A_{s,22} & g \\			
			A_{s,31} & -k_m g^\top & 0 
			\end{bmatrix},
		\end{align}
		and 
		\begin{align*}
			A_{s,21} \coloneqq& -\frac{1}{m}\left(k_{\phi} + k_v + \hat{\alpha}^\star \right) \nabla_x^2\phi(x)\big|_{x^\star} \\
			A_{s,22} \coloneqq& -\nabla_x^2\phi(x)\big|_{x^\star} - (k_v + \hat{\alpha}^\star)I_n - \frac{\partial f(x,v)}{\partial v}\bigg|_{s^\star}.
		\end{align*}		
		We aim to prove that the equilibrium {$s_x^\star \coloneqq [0_n^\top,0_n^\top,0]^\top$} has at least one positive eigenvalue. To this end, consider a vector $\bar{\nu} \coloneqq [\mu \nu^\top, \nu^\top,0]^\top$	, where $\mu > 0$ is a positive constant, and $\nu \in\mathbb{R}^n$ is an orthogonal vector to $(x^\star - c_k)$, i.e. $\nu^\top(x^\star - c_k) = 0$. Then the respective quadratic form yields
		\begin{align*}
			\bar{\nu}^\top A_s \bar{\nu} = & \begin{bmatrix}
			\nu^\top A_{s,21} & \mu \nu^\top + \nu^\top A_{s,22} & \nu^\top g
			\end{bmatrix} \begin{bmatrix}
			\mu \nu \\ 
			\nu \\ 
			0
			\end{bmatrix} = \\
			& 
			\mu \nu^\top A_{s,21} \nu + \mu \|\nu\|^2 + \nu^\top A_{s,22} \nu,		
		\end{align*}
		which, after employing \eqref{eq:grad + hessian} with $\beta'(d_j^\star) = 0$, $\forall j\in\mathcal{J}\backslash\{k\}$ and $\nu^\top(x^\star - c_k) = 0$, becomes		
		\begin{align*}
			&\bar{\nu}^\top A_s \bar{\nu} = -\frac{2\mu k_1}{m}\left(k_{\phi} + k_v + {\frac{3}{2}}\hat{\alpha}^\star\right) \left(1 + \frac{k_2}{k_1} \beta'(d_k^\star) \right)\|\nu\|^2  \\
			&+ \mu\|\nu\|^2 - 2k_1 \left(1 + \frac{k_2}{k_1} \beta'(d_k^\star) \right)\|\nu\|^2 - \left(k_v + {\frac{3}{2}}\hat{\alpha}^\star\right)\|\nu\|^2 -\nu^\top \frac{\partial f(x,v)}{\partial v}\bigg|_{s^\star} \nu.
		\end{align*}		
		From \eqref{eq:grad at equil k in J}, by recalling that $\beta'(d_k) \leq 0$, we obtain that
		\begin{equation} \label{eq:at critical point}
			\frac{k_2}{k_1}\beta'(d_k^\star) = -\frac{\|x^\star - x_{\textup{d}}\|}{\|x^\star - c_k\|}  < -1.
		\end{equation}
		Therefore by defining $c^\star \coloneqq -\frac{k_2}{k_1}\beta'(d_k^\star) -1 > 0$, we obtain
		\begin{align*}
			\bar{\nu}^\top A_s \bar{\nu} =& \Bigg(\frac{2\mu k_1}{m}k_{\phi}c^\star  +  \left(\frac{2\mu k_1}{m}c^\star -1 \right)\left(k_v + {\frac{3}{2}}\hat{\alpha}^\star\right)   + \mu + 2k_1c^\star \Bigg)\|\nu\|^2 \notag \\ & -\nu^\top \frac{\partial f(x,v)}{\partial v}\bigg|_{s^\star} \nu,
		\end{align*}
		which is rendered positive by choosing a sufficiently large $\mu$. Hence, $A_s$ has at least one positive eigenvalue. Next, we prove that $A_s$ has no zero eigenvalues by proving that its determinant is nonzero. For the determinant of $\nabla^2_x\phi(x)|_{x^\star}$, in view of \eqref{eq:grad + hessian} that
		\begin{align*}
			\det(\nabla_x^2\phi(x)|_{x^\star}) =& \det \Bigg(2\left(k_1 + k_2\beta'(d^\star_k)\right)I_n +  
			2k_2 \beta''(d^\star_k) (x^\star-c_k)(x^\star-c_k)^\top \Bigg).
		\end{align*} 
		By using the property $\det(A + uv^\top) = (1+v^\top A^{-1}u) \det(A)$, for any invertible matrix $A$ and vectors $u,v$, we obtain 
		\begin{align} \label{eq:det hessian 1}
			\det(\nabla_x^2\phi(x)|_{x^\star}) =& 2^n \big(k_1 + k_2\beta'(d_k^\star) \big)^n \Bigg(1 + \notag\\
			&\hspace{-10mm}\frac{k_2}{k_1\left(1 + \frac{k_2}{k_1}\beta'(d_k^\star)\right)}\beta''(d_k^\star)\|x^\star-c_k\|^2 \Bigg).
		\end{align}
		In view of \eqref{eq:at critical point} and by using $\|x^\star -x_{\textup{d}}\| - \|x^\star - c_k\| = \|x_{\textup{d}} - c_k\|$ since $x^\star$, $c_k$ and $x_{\textup{d}}$ are collinear,  \eqref{eq:det hessian 1} becomes
		\begin{align*} 
		\det(\nabla_x^2\phi(x)|_{x^\star}) =& 2^n \big(k_1  k_2\beta'(d_k^\star)\big)^n \Bigg(1  
	-\frac{k_2}{k_1 \|x_{\textup{d}} - c_k\|}\beta''(d_k^\star)\|x^\star-c_k\|^3 \Bigg).
		\end{align*}
		Note that, since $\lim_{d_j \to 0}\beta(d_j) = \infty$ and $\beta(d_j)$ decreases to $\beta(d_j) = \beta(\tau)$, $\forall d_j \geq \tau$, the derivatives $\beta'(d_j)$ satisfy $\lim_{d_j\to 0}\beta'(d_j) = -\infty$ and increase to $\beta'(d_j) = 0$, $\forall d_j \geq \tau$. Hence, we conclude that $\beta''(d_j) > 0$, $\forall d_j \in (0,\tau)$. Therefore, in order for the critical point to be non-degenerate, we must guarantee that 
		\begin{equation} \label{eq:det hessian 3}
			\frac{k_2}{k_1 \|x_{\textup{d}} - c_k\|}\beta''(d_k^\star) \|x^\star-c_k\|^3 > 1.
		\end{equation}	
		By expressing $\|x^\star - c_k\|^3 = (d_k^\star + \bar{r}_{o_k}^2 )\sqrt{d_k^\star + \bar{r}_{o_k}^2}$, considering that $\|x_{\textup{d}} - c_k\| \leq 2\bar{r}_{\mathcal{W}}$ and setting $\underline{r} \coloneqq \min_{j\in\mathcal{J}}\{\bar{r}_{o_j}\}$, a lower bound for the left-hand side of \eqref{eq:det hessian 3} is 
		\begin{equation} \label{eq:det hessian 4}
			f_\ell(d_k^\star) \coloneqq \frac{k_2}{2k_1 \bar{r}_{\mathcal{W}} } \beta''(d_k^\star)(d_k(x^\star) + \underline{r}^2 )\sqrt{d_k(x^\star) + \underline{r}^2}.
		\end{equation}
		According to Property $4$ of Definition \ref{def:2nd nf}, \eqref{eq:det hessian 4} is a decreasing function of $d_k^\star$, for $d_k^\star \in (0,\tau)$, with $f_\ell(\tau) = 0$ and $\lim_{d_k^\star \to 0} f_\ell(d_k^\star) = \infty$. Therefore, there exists a positive $d_k^{\star\star} > 0$, such that $f_\ell(d_k^\star) > 1$, $\forall d_k^\star < d_k^{\star\star}$. Hence, by setting $\tau < d_k^{\star\star}$, we achieve $d_k^\star < \tau < d_k^{\star\star}$ and guarantee that $f_\ell(d_k^\star) > 1$.

		By defining $A_{2ns} \coloneqq \begin{bmatrix}
		0_{n\times n} & I_n \\A_{s,21} & A_{s,22}
		\end{bmatrix}$, it holds that		
		\begin{align*}
\det(A_{2ns}) = \det(A_{s,21}) = \frac{(-1)^n}{m^n}\left(k_{\phi} + k_v + {\frac{3}{2}}\hat{\alpha}^\star\right)^n \det(\nabla^2_x \phi(x)|_{x^\star}) \neq 0.
		\end{align*} 
		Moreover, it holds that $A_{2ns}^{-1} = \begin{bmatrix}
		\star & A_{s,21}^{-1} \\
		\star & 0_{n\times n}
		\end{bmatrix}$
		and therefore we obtain that		
		\begin{align*}
			\det(A_s) &=
			\det(A_{s,21})\begin{bmatrix}
			k_m g^\top(\nabla_x^2 \phi(x))^\top|_{x^\star} & k_m g^\top
			\end{bmatrix}  A_{2ns}^{-1} \begin{bmatrix}
			0_n \\ g 
			\end{bmatrix}  \\&= 
			\det(A_{s,21}) \begin{bmatrix}
			k_m g^\top(\nabla_x^2 \phi(x))^\top|_{x^\star} & k_m g^\top
			\end{bmatrix}  \begin{bmatrix}
			A_{s,21}^{-1} \ g \\ 0_n 
			\end{bmatrix}   \\
		 	&= \det(A_{s,21})
		 	k_m g^\top(\nabla_x^2 \phi(x))^\top|_{x^\star}
		 	A_{s,21}^{-1} \ g  \\
		 	&=
		 	k_m g^\top(\nabla_x^2 \phi(x))^\top|_{x^\star}
		 	\textup{adj}(A_{s,21}) \ g,		 	
		\end{align*}
		which is non-zero, since  $$\det(\nabla_x^2(\phi(x))^\top|_{x^\star}\textup{adj}(A_{s,21})) = \det(\nabla_x^2(\phi(x))^\top|_{x^\star}) \det(A_{s,21})^{n-1} \neq 0,$$ and $g \neq 0_n$\footnote{A similar analysis can be performed in the 2-dimensional, where $g=0_n$.}, and hence the matrix that forms the latter quadratic form is nonsingular.
				
		Therefore, we conclude that $A_s$ is non-degenerate and has at least one positive eigenvalue. Note that {$\bar{A}_s$ has the same eigenvalues as $A_s$ and an extra zero eigenvalue}. According to the Reduction Principle  \cite[Th. 5.2]{kuznetsov2013elements}, $\dot{s}_e = \bar{A}_s s_e$ is locally topologically equivalent near the origin to the system
		\begin{align*}			
			\dot{\hat{\alpha}} &= k_\alpha \big\| v_\alpha(\hat{\alpha}) + \nabla_x\phi(x)|_{x_\alpha(\hat{\alpha})}\big\|^2 \\
			\dot{s}_x &= A_s s_x,		
		\end{align*}
		where 
		$v_\alpha(\hat{\alpha})$, $\nabla_x\phi(x)|_{x_\alpha(\hat{\alpha})}$ are the restrictions of $v$ and $\nabla_x\phi(x)$ to the center manifold of $\hat{\alpha}$ {\cite[Theorem 5.2]{kuznetsov2013elements}}. Regarding the trajectories of $s_x$, since $A_s$ is a non-degenerate saddle (it has at least one positive eigenvalue)  its stable manifold has dimension lower than $2n+1$ and is thus a set of zero measure.
		Therefore, all the initial conditions $(x(t_0),v(t_0),\widetilde{m}(t_0)) \in \mathcal{F}\times\mathbb{R}^{n+1}$, except for the aforementioned lower-dimensional manifold, converge to the desired equilibrium $(x_{\textup{d}},0_n,0)$.		
	\end{proof}

\begin{remark}
	{Note that, unlike the related works in {feedback-based robot navigation}, the proposed algorithm guarantees almost global safe convergence while compensating for unknown dynamic terms ($f$ and $m$ in this case). Moreover, in contrast to tuning schemes (e.g.,  {\cite{koditschek1990robot,loizou2011closed,vlantis2018robot,dimarogonas2006feedback}}), we do not require large {control gains} in order to establish the correctness of the propose scheme. }	
\end{remark}

\begin{remark}
	The condition $k_\phi > \frac{\alpha}{2}$ of Theorem \ref{th:single robot} is only sufficient and not necessary, as will be shown in the simulation results. {{Moreover, in case the robot gets stuck in a local minima, one could apply an exciting input perpendicular to $x-x_\textup{d}$ (see \cite{vrohidis2018prescribed}), freeing it thus from that configuration. Nevertheless, the set of initial conditions that drive the robot to such configurations has zero measure and hence the probability of starting in it is zero.\footnote{{{The exciting input could be applied at the initial condition, if it can be identified that it will lead to a local minima.}}}}}
\end{remark}

{\subsection{Dynamic Disturbance Addition}} \label{subsec:disturbance}

{Except for the already considered dynamic uncertainties, we can add to the right-hand side of \eqref{eq:dynamics} an unknown disturbance vector $d(x,v,t)$, { i.e.,
	\begin{align*}
	& \dot{x}= v \\
	& m \dot{v} + f(x,v) + mg + d(x,v,t)= u,
	\end{align*}}
subject to a uniform boundedness condition $\|d(x,v,t)\|$ $\leq$ $\bar{d}$, $\forall x,v,t\in\mathbb{R}^{2n}\times\mathbb{R}_{\geq 0}$. In this case, by slightly modifying the control scheme, we still guarantee collision avoidance with the workspace obstacles and boundary. In addition, we 
achieve uniform ultimate boundedness of the error signals as well as the gradient of $\phi$, as the analysis in this section shows.}

{ 
	The control scheme of the previous section is appropriately enhanced to incorporate the $\sigma$- modification \cite{lavretsky2013robust}, a common technique in adaptive control. More specifically, the adaptation laws \eqref{eq:adaptation laws} are modified according to 
	\begin{align*}
		\dot{\hat{m}} \coloneqq& -k_m e_v^\top( \dot{v}_{\textup{d}}+g) - \sigma_m\hat{m}  \\
		\dot{\hat{\alpha}} \coloneqq & k_\alpha \|e_v\|^2 - \sigma_\alpha\hat{\alpha}, 
	\end{align*}
	where $\sigma_m$, $\sigma_\alpha$ are positive gain constants, {to be appropriately tuned as per the analysis below}. }

{
	Consider now the function $V$ as defined \eqref{eq:V Lyap}. In view of the analysis of the previous section, the incorporation of $d(x,v,t)$, as well as the modification of the adaptation laws, the derivative of $V$ becomes 
	\begin{align*}
		\dot{V} \leq& -(k_\phi - \tfrac{\alpha}{2})\|\nabla_x \phi(x)\|^2 - k_v\|e_v\|^2 + \|e_v\|\bar{d} \\
		&- {\tfrac{3}{2}}\sigma_\alpha \widetilde{\alpha}\hat{\alpha} - \sigma_m\widetilde{m}\hat{m},
	\end{align*} 
	which, by using $\hat{\alpha} = \widetilde{\alpha} + \alpha$, $\hat{m} = \widetilde{m} + m$, as well as the properties $-ab = -\frac{1}{2}(a+b)^2 + \frac{a^2}{2} + \frac{b^2}{2}$, $ab = -\frac{1}{2}(a-b)^2 + \frac{a^2}{2} + \frac{b^2}{2}$, $\forall a,b \in \mathbb{R}$, becomes

	\small
	\begin{align*}
		\dot{V} \leq&  -(k_\phi - \frac{\alpha}{2})\|\nabla_x \phi(x)\|^2 - (k_v-\frac{1}{2})\|e_v\|^2 + \frac{\bar{d}^2}{2}- {3\sigma_\alpha\frac{\widetilde{\alpha}^2}{4}} \\ & - \sigma_m\frac{\widetilde{m}^2}{2} + {3\sigma_\alpha\frac{\alpha^2}{4}} + \sigma_m\frac{m^2}{2} \leq -k_\xi \|\xi\|^2 + d_\xi, 		
	\end{align*}
	\normalsize
	where $\xi \coloneqq [\nabla_x\phi(x)^\top, e_v^\top, \widetilde{m}, \widetilde{\alpha}]^\top \in \mathbb{R}^{2n+2}$, $k_\xi \coloneqq \min\{k_\phi-\frac{\alpha}{2}, k_v -\frac{1}{2}, \frac{\sigma_m}{2},{\frac{3\sigma_\alpha}{4}}\}$, and
	$d_\xi \coloneqq \frac{\bar{d}^2}{2} + {3\sigma_\alpha\frac{\alpha^2}{4}} + \sigma_m\frac{m^2}{2}$. {Therefore, $\dot{V}$ is negative when $\|\xi\| > {\sqrt\frac{d_\xi}{k_\xi}}$, which, by also requiring $k_v > \frac{1}{2}$, implies that $$\|\xi(t)\| \leq \max\left\{ \gamma_1(\|\xi(t_0)\|), \gamma_2\left(\sqrt{\frac{d_\xi}{k_\xi}}\right) \right\},$$ $\forall t \geq t_0$, where $\gamma_i$ are class $K$ functions \cite[Th. 4.18]{khalil1996noninear}. Since $\xi(t)$, and hence $\nabla_x \phi(x)$, {remain} bounded, collisions are avoided.}
	
{
	Note that the aforementioned analysis guarantees that $\nabla_x\phi(x)|_{x(t)}$ will be ultimately bounded in a set close to zero. This point, however, might be a critical point of $\phi$ and it is not guaranteed that $x(t)$ will be bounded close to the goal configuration $x_\textup{d}$. Nevertheless, intuition suggests that if the disturbance vector $d(x,v,t)$ does not behave adversarially, the agent will converge close to the goal configuration. This is also verified by the simulation results of Section \ref{sec:sim}. 
 }  

\section{Extension to Star Worlds} \label{sec:Star}
In this section, we discuss how the proposed control scheme can be extended to generalized sphere worlds, and in particular \textit{star worlds}, {{being inspired by} the methodology of \cite{rimon1992exact}. That work however, like others related to workspace transformations \cite{loizou2014multi,vlantis2018robot}, consider simplified dynamics without taking into account unknown terms, which is the focus of this section. Although we focus on star-worlds, the analysis holds for any differeomorphic transformation that exhibits the desired properties (e.g. \cite{vlantis2018robot})}.  \textit{Star worlds} are diffeomorphic to sphere worlds sets of the form 
$\mathcal{T} \coloneqq \bar{\mathcal{W}} \backslash \bigcup_{j\in\mathcal{J}} {\bar{O}_{\mathcal{T}_j}}$,
where $\bar{\mathcal{W}}$ is a workspace of the form \eqref{eq:transf. workspace} and $\bar{O}_{\mathcal{T}_j}$ are $M$ disjoint star-shaped obstacles (indexed by $\mathcal{J} = \{1,\dots,M\}$). The latter are sets characterized by the  property that all rays emanating from a center point cross their boundary only once {\cite{rimon1992exact}}.
One can design a diffeomorphic mapping $H:\mathcal{T} \to \mathcal{F}$, where $\mathcal{F}$ is a sphere world of the type \eqref{eq:sphere world}. More specifically, $H$ maps the boundary of $\mathcal{T}$ to the boundary of $\mathcal{F}$. Construction of such a mapping is beyond the scope of the paper and we refer the interested reader to the related literature {\cite{rimon1992exact,rimon1991construction}}. 

The control scheme of the previous section is modified now to account for the transformation $H$ as follows. The desired robot velocity is set to $v_{\textup{d}}: \mathcal{T} \to \mathbb{R}^n$, with
\begin{equation} \label{eq:v_des star}
	v_{\textup{d}}(x) \coloneqq - J_H(x)^{-1} \nabla_{H(x)}\phi(H(x)),
\end{equation}
where $J_H(x) \coloneqq \frac{\partial H(x)}{\partial x}$ is the nonsingular Jacobian matrix of $H$. Next, by letting $e_v \coloneqq v - v_{\textup{d}}$, the control law is designed as $u:\mathcal{T}\times\mathbb{R}^{n+2} \to \mathbb{R}^n$, with
\begin{align} \label{eq:u star}
	u\coloneqq u(x,v,\hat{m},\hat{\alpha}) \coloneqq& -k_\phi J_h(x)^\top \nabla_{H(x)} \phi(H(x)) + \notag \\ 
	& \hat{m}(\dot{v}_{\textup{d}} + g) - \left(k_v + {\frac{3}{2}}\hat{\alpha}\right)e_v,
\end{align}
where $\hat{m}$ and $\hat{\alpha}$ evolve according to the respective expressions in \eqref{eq:adaptation laws}. The next theorem gives the main result of this section. 
\begin{theorem}\label{th:single robot star worlds}
	Consider a robot operating in $\mathcal{W}$, subject to the uncertain $2$nd-order dynamics \eqref{eq:dynamics}. Given $x_{\textup{d}} \in \mathcal{T}$, the control protocol {\eqref{eq:v_des star},  \eqref{eq:u star}, \eqref{eq:adaptation laws}}   guarantees the collision-free navigation to $x_{\textup{d}}$ from almost all initial conditions $(x(t_0),v(t_0),\hat{m}(t_0),\hat{\alpha}(t_0)) \in \mathcal{T}\times\mathbb{R}^{n_1}\times\mathbb{R}_{\geq 0}$, given a sufficiently small $\tau$ and that $k_\phi > \frac{\alpha}{2}$. Moreover, all closed loop signals remain bounded, $\forall t \geq t_0$.
\end{theorem}  

\begin{proof}
Following similar steps as in the proof of Theorem \ref{th:single robot}, we consider the function 
\begin{equation}
V \coloneqq k_{\phi}\phi(H(x)) + \frac{m}{2}\|e_v\|^2 + \frac{1}{2k_\alpha}\widetilde{\alpha}^2 + \frac{1}{2k_m}\widetilde{m}^2,
\end{equation}	
whose derivative along the solutions of the closed loop system yields 
\begin{equation}
	\dot{V} \leq -\left(k_\phi - \frac{\alpha}{2}\right) \|\nabla_{H(x)} \phi(H(x)) \|^2 - k_v\|e_v\|^2 \leq 0,
\end{equation}
which proves the boundedness of the obstacle functions $\beta(d_j(H(x(t))))$, $\forall j\in\mathcal{J}, t\geq t_0$. Since the boundaries $\partial \bar{\mathcal{O}}_j$ are mapped to $\partial \bar{\mathcal{O}}_{\mathcal{T}_j}$ through $H(x)$, we conclude that $x(t) \in \mathcal{T}$, $t\geq t_0$ and no collisions occur. Next, by following similar arguments as in the proof of Theorem \ref{th:single robot}, we conclude that the solution will converge to a critical point of $\phi(H(x))$. By choosing a sufficiently small $\tau$ for the obstacle functions $\beta(d_j(H(x(t))))$, the critical points consist of the desired equilibrium, where $\beta'(d_j(H(x_{\textup{d}}))) = 0$, $\forall j\in\mathcal{J}$, or undesired critical points $x^\star$ satisfying 

\small
\begin{equation} \label{eq:critical point star world}
	k_1(H(x^\star) - H(x_{\textup{d}})) = -k_2 \beta'(d_{H_k}^\star)(H(x^\star) - H(c_k)),
\end{equation}
\normalsize
for some $k\in\mathcal{J}$, where we define $d_{H_j}^\star \coloneqq d_j(H(x^\star))$, $\forall j\in\mathcal{J}$. The respective terms of the linearization matrix $\bar{A}_s$ from \eqref{eq:s_e dot} become now	
\begin{align*}
\bar{A}_s \coloneqq& \begin{bmatrix}
A_s & 0_{2n+2} \\ 
0_{2n+2}^\top & 0
\end{bmatrix} \\
A_s \coloneqq& \begin{bmatrix}
0_{n\times n} & I_n & 0_n \\
A_{s,21} & A_{s,22} & g \\			
A_{s,31} & -k_m g^\top & 0 
\end{bmatrix},
\end{align*}
with 
\begin{align*}
A_{s,21} \coloneqq& -\frac{1}{m}\big( k_{\phi}J_H(x^\star)^\top + \left(k_v + {\frac{3}{2}}\hat{\alpha}^\star\right)J_H(x^\star)^{-1} \big) \nabla^2 \phi^\star J_H(x^\star) \\
A_{s,22} \coloneqq& -J_H(x^\star)^{-1} \nabla^2 \phi^\star J_H(x^\star) - \left(k_v + {\frac{3}{2}}\hat{\alpha}^\star\right)I_n - \frac{\partial f(x,v)}{\partial v}\bigg|_{s^\star}, \\ 
A_{s,31} \coloneqq & -k_m g^\top \big(J_H(x^\star)^{-1} \nabla^2 \phi^\star J_H(x^\star)\big)^\top
\end{align*}
and $\nabla^2 \phi^\star \coloneqq \nabla^2_{H(x)}\phi(H(x))|_{x^\star}$,
around $x = x^\star, v = 0, \widetilde{m} = 0, \widetilde{\alpha} = \widetilde{\alpha}^\star$. Next, similarly to the proof of Theorem \ref{th:single robot star worlds}, we prove that $\bar{\nu}^\top A_s \bar{\nu} > 0$, for $\bar{\nu} \coloneqq [\mu \nu^\top, \nu^\top, 0]^\top$, where $\mu >0$ is a positive constant and $\nu \coloneqq J_H(x^\star)^{-1}\hat{\nu}$, with
$\hat{\nu}\in\mathbb{R}^n$ a vector orthogonal to $(H(x^\star) - H(c_k))$.
The respective quadratic form yields, after employing \eqref{eq:critical point star world} and defining $c^\star \coloneqq -\left(1 +  \frac{k_2}{k_1}\beta'(d^\star_{H_k}) \right) > 0$:

\small
\begin{align*}
	\bar{\nu}^\top A_s \bar{\nu} =& 	\hat{\nu}^\top \Bigg[ \frac{2k_1k_{\phi}\mu c^\star}{m} I_n + J_H(x)^{-\top} \Bigg( \bigg(\frac{2k_1 c^\star(k_v + {\frac{3}{2}}\hat{\alpha}^\star)}{m}  + \\
	&\hspace{-10mm} \mu - \left(k_v + {\frac{3}{2}}\hat{\alpha}^\star\right) + 2k_1 c^\star \bigg) I_n - \frac{\partial f(x,v)}{\partial v } \bigg|_{s^\star}   \Bigg) J_H(x)^{-1} 
	\Bigg] \hat{\nu},
\end{align*}
\normalsize
which can be rendered positive for sufficiently large $\mu$.

Moreover, at a critical point $x^{\star,1}$ of $\phi(H(x))$, it holds that (see the proof of Prop. 2.6 in {\cite{koditschek1990robot}}),
\begin{equation*}
	\nabla^2_{H(x)} \phi(H(x))|_{x^{\star,1}} = J_H(x^{\star,1})^\top \nabla^2_x \phi(x) |_{x^{\star,2}} J_H(x^{\star,1}),
\end{equation*} 
where $x^{\star,2}$ is a critical point of $\phi(x)$ satisfying $x^{\star,2} = H(x^{\star,1})$. Since $J_H(x)$ is nonsingular, it holds that $x^{\star,1}$ is non-degenerate if and only if $x^{\star,2}$ is non-degenerate. As already shown in the proof of Theorem \ref{th:single robot}, by choosing $\tau$ sufficiently small, we render the critical points of $\phi(x)$ that are close to the obstacles non-degenerate. 
Hence, we conclude that the respective critical points of $\phi(H(x))$ are also non-degenerate and $\det(\nabla^2 \phi^\star) \neq 0$.

Next, in order to prove that the critical point $(x^\star,0,0)$ is non-degenerate, we calculate the determinant of $A_s$. Following the proof of Theorem \ref{th:single robot}, we obtain that 
\begin{align*}
	\det(A_s) &= \det(A_{s,21}) k_m g^\top \big(J_H(x^\star)^{-1} \nabla^2 \phi^\star J_H(x^\star)\big)^\top A_{s,21}^{-1}g \\
	&= k_m g^\top \big(J_H(x^\star)^{-1} \nabla^2 \phi^\star J_H(x^\star)\big)^\top \textup{adj}(A_{s,21})g
\end{align*}
where 
\begin{align*}
	\det(A_{s,21}) =&(-1)^n\Bigg( \frac{k^n_{\phi}}{m^n}\det\left(J_H(x^\star)\right) + \\
	& \hspace{-10mm} \left(k_v+{\frac{3}{2}}\hat{\alpha}^\star\right)^n \frac{1}{\det(J_H(x^\star))} \Bigg)\det(\nabla^2\phi^\star)\det(J_H(x^\star)),
\end{align*}
\normalsize
which is not zero, since $\det(\nabla^2 \phi^\star) \neq 0$ and $J_H(x^\star) \neq 0$. Hence, we conclude that the aforementioned quadratic form is also not zero and hence the non-degeneracy of the critical points under consideration. Hence, by following similar arguments as in the proof of Theorem \ref{th:single robot}, we conclude that the initial conditions that converge to these critical saddle points form a set of measure zero. 
\end{proof}
	 
\begin{remark}
	{The proposed schemes can {also be} extended to \textit{unknown} environments, where the amount and location of the  obstacles is unknown a priori, and these are sensed locally on-line. In particular, by having a large enough sensing neighborhood, each obstacle $j\in\mathcal{J}$ can be sensed when $d_j = \tau$, and hence the respective term can be smoothly incorporated in $\nabla_x \phi(x)$, in view of the properties of $\beta$ (a similar idea is discussed in Section V of \cite{paternain2017navigation}). It should be noted, however, that the local sensory information and respective hardware must allow for the accurate estimation of the centers and radii (or the implicit function in case of star-worlds) of the obstacles.}
\end{remark}

\section{Extension to Multi-Robot Systems} \label{sec:MAS}

This section is devoted to extending the results of Section \ref{sec:main} to {multi-robot} systems. Consider, therefore, $N\in\mathbb{N}$ spherical robots operating in a workspace $\mathcal{W}$ of the form \eqref{eq:workspace}, characterized by their position vectors $x_i\in\mathbb{R}^n$, as well as their radii $r_i > 0$, $i\in\mathcal{N} \coloneqq \{1,\dots,N\}$, and obeying the second-order uncertain dynamics \eqref{eq:dynamics}, i.e.,
\begin{subequations} \label{eq:dynamics MAS}
\begin{align} 
& \dot{x_i}= v_i \\
& m_i \dot{v_i} + f_i(x_i,v_i) + m_ig = u_i,
\end{align}
\end{subequations}
with {the unknown} $f_i(\cdot)$ satisfying $\|f_i(x_i,v_i)\| \leq \alpha_i\|v_i\|$, for unknown positive constants $\alpha_i$, $\forall i\in\mathcal{N}$. We also denote $x\coloneqq [x_1^\top,\dots,x_N^\top]^\top$, $v\coloneqq [v_1^\top,\dots,v_N^\top]^\top \in\mathbb{R}^{Nn}$. Each robot's destination is $x_{\textup{d}_i}$, $i\in\mathcal{N}$.

The proposed multi-robot scheme is based on a prioritized leader-follower coordination.  Prioritization in multi-agent systems for {navigation-type} objectives has been employed in \cite{roussos2013decentralized} and \cite{guo2016communication}, where KRNF gain tuning-type methodologies are developed. {The proposed framework, however, is substantially different from these works; \cite{guo2016communication} does not take into account inter-agent collisions, and uses prioritization for the sequential navigation and task satisfaction subject to connectivity constraints, {while} \cite{roussos2013decentralized} uses prioritization for directional collision-avoidance. In our proposed prioritized leader-follower methodology, the leader robot, by appropriately choosing the offset $\tau$, ``sees" the other robots as static obstacles and hence the overall scheme reduces to the one of Section \ref{sec:main}. This is accomplished by differentiating the free spaces of the leader and the followers. Moreover, the aforementioned works {\cite{guo2016communication,roussos2013decentralized}} consider simplified first-order dynamics and cannot be easily extended to the uncertain dynamics-case considered here. In fact, we note that, according to our best knowledge, there does not exist a control framework that provably guarantees decentralized safe multi-robot navigation in workspaces with obstacles and subject to uncertain $2$nd-order dynamics.}

The workspace is assumed to satisfy Assumption \ref{ass: workspace} and we further impose the following extra conditions:

\begin{assumption} \label{ass:MAS workspace}
	The workspace $\mathcal{W}$, obstacles $\mathcal{O}_j$, $j\in\mathcal{J}$, and destinations {$x_{\textup{d}_i}$, $i\in\mathcal{N}$,} satisfy:	
	\begin{align*}			
		&\|c_j - x_{\textup{d}_i}\| > r_{o_j} + r_i + 2r_M + \varepsilon, \forall i,j\in\mathcal{N}\times\mathcal{J}  \\
		&\|x_{\textup{d}_i} - x_{\textup{d}_j}\| > r_i + r_j + 2r_M + 2\varepsilon, \forall i,j\in\mathcal{N}, i\neq j \\		
		&r_\mathcal{W} - \|x_{\textup{d}_i}\| > r_i + 2r_M + \varepsilon, \forall i\in\mathcal{N}
	\end{align*}	
	whereas the initial positions satisfy:
	\begin{align*}
		&\|c_j - x_i(t_0)\| > r_{o_j} + r_i + 2r_M, \forall i,j\in\mathcal{N}\times\mathcal{J}  \\
		&r_\mathcal{W} - \|x_i(t_0)\| > r_i + 2r_M, \forall i\in\mathcal{N}  \\
		&\|x_{\textup{d}_i} - x_j(t_0) \| > r_i + r_j + 2r_M + \varepsilon, \forall i,j\in\mathcal{N}, i \neq j, 
	\end{align*}
	for an arbitrarily small positive constant  $\varepsilon$, $\forall i\in\mathcal{N}, j\in\mathcal{J}$, where $r_M \coloneqq \max_{i\in\mathcal{N}}\{r_i\}$. 
\end{assumption}

Loosely speaking, the aforementioned assumption states that the pairwise distances among obstacles, workspace boundary, initial conditions and final destinations are large enough so that one robot can always navigate between them. {Since} the convergence of the agents to the their destinations is asymptotic, we incorporate the threshold $\varepsilon$, which is the desired proximity we want to achieve to the destination, as will be clarified in the sequel. 
{Intuitively, since we cannot achieve $x_i = x_{\textup{d}_i}$ in finite time, the high-priority agents will stop once $\|x_i-x_{\textup{d}_i}\|= \varepsilon$, which is  included in the aforementioned conditions to guarantee the feasibility of the collision-free navigation for the lower-priority agents.}
Similarly to the single-agent case, we can find a positive constant $\bar{r}$ such that \eqref{eq:r_bar} hold as well as
\begin{subequations} \label{eq:r_bar_mas}
\begin{align}			
&\|c_j - x_i(t_0)\| > r_{o_j} + r_i + 2r_M + 2\bar{r}, \forall i,j\in\mathcal{N}\times\mathcal{J}  \\
&r_\mathcal{W} - \|x_i(t_0)\| > r_i + 2r_M + 2\bar{r}, \forall i\in\mathcal{N} \\
&\|c_j - x_{\textup{d}_i}\| > r_{o_j} + r_i + 2r_M + \varepsilon + 2\bar{r}, \forall i,j\in\mathcal{N}\times\mathcal{J}  \\
&\|x_{\textup{d}_i} - x_{\textup{d}_j}\| > r_i + r_j + 2r_M + 2\varepsilon + 2\bar{r}, \forall i,j\in\mathcal{N}, i\neq j \\
&\|x_{\textup{d}_i} - x_j(t_0) \| > r_i + r_j + 2r_M + \varepsilon + 2\bar{r}, \forall i,j\in\mathcal{N}, i \neq j \\
&r_\mathcal{W} - \|x_{\textup{d}_i}\| > r_i + 2r_M + \varepsilon + 2\bar{r}, \forall i\in\mathcal{N}
\end{align}
\end{subequations}
We consider that the agents have a limited sensing range, defined by a radius ${\varsigma_i} > 0$, $i\in\mathcal{N}$, and we assume that each agent $i$ can sense the state of its neighbors: 
\begin{assumption} \label{ass:sensing }
	Each agent $i\in\mathcal{N}$ has a limited sensing radius $\varsigma_i$, satisfying $\varsigma_i > \sqrt{\min(\bar{r}^2,\bar{r}_\textup{d})} + r_i + r_j +2r_M + 2\bar{r}$, with $\bar{r}_\textup{d}$ as defined in \eqref{eq:bar_r_d}, and has access to {$x_i-x_j,v_i-v_j$, $\forall j\in\{ j\in\mathcal{N} : \|x_i-x_j\| \leq \varsigma_i \}$}.
\end{assumption} 

Moreover, we consider {that} the destinations, $x_{\textup{d}_i}$, $i\in\mathcal{N}$, {as well as the radii, $r_i$}, are transmitted off-line to all the agents{\footnote{{This implies that the agents can compute $r_M$ offline.}}}.
Consider now a prioritization of the agents, possibly based on some desired metric (e.g., distance to their destinations), which can be performed off-line and transmitted to all the agents.
Our proposed scheme is based on the following algorithm. The agent with the highest priority is designated as the leader of the multi-agent system, indexed by $i_\mathcal{L}$, whereas the rest of the agents are considered as the followers, defined by the set $\mathcal{N}_\mathcal{F} \coloneqq \mathcal{N} \backslash \{i_\mathcal{L}\}$. The followers and leader employ a control protocol that has the same structure as the one of Section \ref{sec:main}. The key difference here lies in the definition of the free space for followers and leaders. Let $q = [q_1^\top,\dots,q_N^\top]^\top \in\mathbb{R}^{Nn}$. We define first the sets
\begin{align*}
	&\bar{\mathcal{W}}_{i_\mathcal{L}} \coloneqq \{q \in \mathbb{R}^{Nn}: \|q_{i_\mathcal{L}}\| < r_\mathcal{W} - r_{i_\mathcal{L}} \},  \\
	&\bar{\mathcal{O}}_{i_\mathcal{L},j} \coloneqq \{ q \in {\bar{\mathcal{W}}_{i_\mathcal{L}}} : \|q_i-c_j\| \leq r_{o_j} + r_i \}, \forall j\in\mathcal{J} \\
	& \mathcal{C}_{i_\mathcal{L}} \coloneqq \{ q \in {\bar{\mathcal{W}}_{i_\mathcal{L}}}: \|q_{i_\mathcal{L}} - q_j \| \leq r_{i_\mathcal{L}} + r_j, \forall j\in\mathcal{N} \backslash \{i_\mathcal{L}\}  \},  
\end{align*}
which correspond to the leader agent, as well as the follower sets
\begin{align*}
&\bar{\mathcal{W}}_i \coloneqq \{q \in \mathbb{R}^{Nn}: \|q_i\| < r_\mathcal{W} - r_i - 2r_M - 2\bar{r} \} \\
&\bar{\mathcal{O}}_{i,j} \coloneqq \{ q \in \bar{\mathcal{W}}_i : \|q_i-c_j\| \leq r_{o_j} + r_i + 2r_M + 2\bar{r} \}, \forall j\in\mathcal{J} \\
& \mathcal{C}_i \coloneqq \{ q \in \bar{\mathcal{W}}_i: \|q_i - q_{i_\mathcal{L}} \| \leq r_i + r_{i_\mathcal{L}},  \\
&\hspace{12mm}  \|q_i - q_j \| \leq r_i + r_j + 2r_M + 2\bar{r}, \forall j\in\mathcal{N} \backslash \{i_\mathcal{L}{,i}\}, \\
&\hspace{12mm}  \|q_i-x_{\textup{d}_j}\| \leq r_i + r_j + 2r_M + 2\bar{r} + \varepsilon, \forall {j\in\mathcal{N}_i}  \}, 
\end{align*}
$\forall i\in\mathcal{N}_\mathcal{F}$, {where $\mathcal{N}_i$ denotes the set of agents with higher priority than agent $i$}. The free space for the agents is defined then as $\mathcal{F}_i \coloneqq \bar{\mathcal{W}}_i \backslash \{ ( \bigcup_{j\in\mathcal{J}} \bar{\mathcal{O}}_{i,j} ) \cup  \mathcal{C}_i \}, \forall i\in\mathcal{N}$.
It can be verified that, in view of \eqref{eq:r_bar_mas}, the sets $\mathcal{F}_i$ are nonempty and  $x(t_0) \in \mathcal{F}\coloneqq \bigcap_{i\in\mathcal{N}}\mathcal{F}_i$. 
The main difference lies in the fact that the follower agents aim to keep a larger distance from each other, the obstacles, and the workspace boundary than the leader agent, and in particular, a distance enhanced by $2r_M + 2\bar{r}$. In that way, the leader agent will be able to choose an  appropriate constant $\tau$ (as in the single-agent case of Section \ref{sec:main}) so that  it is influenced at each time instant only by one of the obstacles/followers, and will be also able to navigate among the obstacles/followers. {Note that the followers are required to stay away also  from the destinations of the higher priority agents}, since a potential local {minimum} in such configurations can prevent the leader agent from reaching its goal. 
We provide next the mathematical details of the aforementioned reasoning.
  
Consider the leader distances $d_{i_\mathcal{L},o_k}$, $d_{i_\mathcal{L},j}$, $d_{i_\mathcal{L},o_0}:\mathcal{F}_{i_\mathcal{L}} \to \mathbb{R}_{\geq 0}$ as 

\begin{align*}
	& d_{i_\mathcal{L},o_k}(x) \coloneqq \|x_{i_\mathcal{L}} - c_k \|^2 - (r_{i_\mathcal{L}} + r_{o_k})^2, \forall k\in\mathcal{J} \\
	& d_{i_\mathcal{L},j}(x) \coloneqq \|x_{i_\mathcal{L}} - x_j \|^2 - (r_{i_\mathcal{L}} + r_j)^2, \forall j\in\mathcal{N}_\mathcal{F} \\
	& d_{i_\mathcal{L},o_0}(x) \coloneqq (r_\mathcal{W}+r_{i_\mathcal{L}})^2 -  \|x_{i_\mathcal{L}}\|^2
\end{align*}
and the follower distances $d_{i,o_k}$, $d_{i,i_\mathcal{L}}$, $d_{i,j}$, $d_{i,\textup{d}_j}$ $d_{i,o_0}:\mathcal{F}_i \to \mathbb{R}_{\geq 0}$ as 

\small
\begin{align*}
& d_{i,o_k}(x) \coloneqq \|x_i - c_k \|^2 - (r_i + r_{o_k} + 2r_M + 2\bar{r})^2, \forall k\in\mathcal{J} \\
& d_{i,i_\mathcal{L}}(x) \coloneqq \|x_i - x_{i_\mathcal{L}} \|^2 - (r_i + r_{i_\mathcal{L}})^2 = d_{i_\mathcal{L},i}(x) \\
& d_{i,j}(x) \coloneqq \|x_i - x_j \|^2 - (r_i + r_j + 2r_M + 2\bar{r})^2, \forall j\in\mathcal{N}_\mathcal{F}\backslash\{i\} \\
& d_{i,\textup{d}_j}(x) \coloneqq \|x_i - x_{\textup{d}_j} \|^2 - (r_i + r_j + 2r_M + 2\bar{r} + \varepsilon)^2,  \forall {j\in\mathcal{N}_i} \\
& d_{i,o_0}(x) \coloneqq (r_\mathcal{W} - r_i - 2r_M - 2\bar{r})^2 -  \|x_i\|^2,
\end{align*}
\normalsize
$\forall i\in\mathcal{N}_\mathcal{F}$. Note that $d_{i,j}(x) = d_{j,i}(x)$, $\forall i,j\in\mathcal{N}$, with $i \neq j$ and also that $x\in\mathcal{F}$ is equivalent to all the aforementioned distances being positive.

Let now functions $\beta$, $\beta_i$, $i\in\mathcal{N}$, that satisfy the properties of Definition \ref{def:2nd nf}, as well as the respective constants $\tau$, $\tau_i$, such that $\beta'({z})=\beta''({z}) =0$, $\forall {z}\geq \tau$, $\beta'_i({z}) = \beta''_i({z}) = 0$, $\forall {z} \geq \tau_i$, $i\in\mathcal{N}$. The $2$nd-order navigation functions for the agents {are} now defined as $\phi_i:\mathcal{F}_i \to \mathbb{R}_{\geq 0}$,  with
\begin{align*}
	\phi_i(x) &\coloneqq k_{1_i}\|x_i - x_{\textup{d}_i}\|^2 + k_{2_i}\bigg(b_{1_i}(x) + b_{2_i}(x) + k_{f_i}b_{3_i}(x)\bigg) \\
	b_{1_i}(x)&\coloneqq \sum_{j\in\bar{\mathcal{J}}} \beta_i( d_{i,o_j}(x) ) , \\
	b_{2_i}(x)&\coloneqq \sum_{j\in\mathcal{N}\backslash\{i\}}  \beta(d_{i,j}(x)) \\
	b_{3_i}(x)&\coloneqq \sum_{{j\in\mathcal{N}_i}}  \beta_i(d_{i,{\textup{d}_j}}(x)),
\end{align*}
$\forall i\in\mathcal{N}$, and $k_{f_{i_\mathcal{L}}} = 0$, $k_{f_i} = 1$, $\forall i\in\mathcal{N}_\mathcal{F}$. Note that the robotic agents can choose independently their $\tau_i$, $i\in\mathcal{N}$, that concerns the collision avoidance with the obstacles and the workspace boundary. The pair-wise inter-agent distances, however, are required to be the same and hence the same $\beta$ (and hence $\tau$) is chosen (see the terms $b_{2_i}(x)$ in $\phi_i(x)$), {which can, nevertheless, be done off-line.}
To achieve convergence of the leader to its destination, we choose $\tau$ and $\tau_{i_\mathcal{L}}$ as in Section \ref{sec:main}, i.e., $\tau, \tau_{i_\mathcal{L}} \in (0, \min\{\bar{r}^2,\bar{r}_\textup{d}\})$. 
Regarding the ability of the agents to sense each other when $d_{i,j}(x) < \tau$,  it holds that
\begin{align*}
	&d_{i,j}(x) < \tau \Leftrightarrow \|x_i - x_j\|^2 \leq \tau + (r_i + r_j + 2r_M + 2\bar{r})^2 \Rightarrow \\
	&\|x_i - x_j\| \leq \sqrt{\tau} + r_i + r_j + 2r_M + 2\bar{r} \Rightarrow \\
	&\|x_i - x_j\| \leq \sqrt{\min\{\bar{r}^2,\bar{r}_\textup{d}\}} + r_i + r_j + 2r_M + 2\bar{r} < \varsigma_i,
\end{align*}
$\forall i,j\in\mathcal{N}$, $i\neq j$, as dictated by Assumption \ref{ass:sensing }. 

The control protocol follows the same structure as the single-agent case presented in Section \ref{sec:main}. In particular, we define the reference velocities as $v_{\textup{d}_i}:\mathcal{F}_i\to\mathbb{R}^n$, with
\begin{equation} \label{eq:v_d MAS}
	v_{\textup{d}_i}(x) \coloneqq -\nabla_{x_i}\widetilde{\phi}_i(x),
\end{equation}
where $\widetilde{\phi}_i:\mathcal{F}_i\to\mathbb{R}_{\geq 0}$ is the slightly modified function:
\begin{align*}
\widetilde{\phi}_i(x) &\coloneqq k_{1_i}\|x_i - x_{\textup{d}_i}\|^2 + k_{2_i}(b_{1_i}(x) + 2b_{2_i}(x) + k_{f_i}b_{3_i}(x))
\end{align*}
{The need for modification of $\phi_i$ to $\widetilde{\phi}_i$ stems from the  differentiation of the terms $b_{2_i}$.}
The control law is designed as $u_i\coloneqq u_i(x,v,\hat{m}_i,\hat{\alpha}_i):\mathcal{F}_i\times\mathbb{R}^{Nn+2} \to \mathbb{R}^n$, with
\begin{align} \label{eq:u_MAS}
\hspace{-1mm}	u_i \coloneqq & -k_{\phi_i} \nabla_{x_i}\widetilde{\phi}_i(x) + \hat{m}_i(\dot{v}_{\textup{d}_i} + g) - \left(k_{v_i} + {\frac{3}{2}}\hat{\alpha}_i\right)e_{v_i}
\end{align}
$\forall i\in\mathcal{N}$; $k_{\phi_i}$, $k_{v_i}$ are positive constants, $e_{v_i}$ are the velocity errors $e_{v_i} \coloneqq v_i - v_{\textup{d}_i}$, and $\hat{m}_i$, $\hat{\alpha}_i$ denote the estimates of $m_i$ and $\alpha_i$, respectively, by agent $i$, {evolving according to \eqref{eq:adaptation laws}.}
We further denote $\hat{m} \coloneqq [\hat{m}_1,\dots,\hat{m}_N]^\top$, $\hat{\alpha} \coloneqq [\hat{\alpha}_1,\dots,\hat{\alpha}_N]^\top \in \mathbb{R}^{N}$.
%
The following theorem considers the convergence of a leader to its destination.

\begin{theorem}\label{th:MAS}
	Consider $N$ robots operating in $\mathcal{W}$, subject to the uncertain $2$nd-order dynamics \eqref{eq:dynamics MAS}, and a leader $i_\mathcal{L}$. 
	Under Assumptions \ref{ass:f}-\ref{ass:sensing }, the control protocol \eqref{eq:v_d MAS}, \eqref{eq:u_MAS}, \eqref{eq:adaptation laws} guarantees collision avoidance between the agents and the agents and obstacles/workspace boundary as well as convergence of $x_{i_\mathcal{L}}$ to $x_{\textup{d}_{i_\mathcal{L}}}$ from almost all initial conditions $(x(t_0),v(t_0),\hat{m}(t_0),\hat{\alpha}(t_0)) \in \mathcal{F}\times\mathbb{R}^{N(n+1)}\times\mathbb{R}^N_{\geq 0}$, given sufficiently small $\tau$, $\tau_{i_\mathcal{L}}$, and that $k_{\phi_i} > \frac{\alpha_i}{2}$, $i\in\mathcal{N}$. Moreover, all closed loop signals remain bounded, $\forall t \geq t_0$.
\end{theorem} 
\begin{proof} 
%
%
	We prove first the avoidance of collisions by considering the function
	\begin{equation*}
	V \coloneqq \sum_{i\in\mathcal{N}}\bigg\{k_{\phi_i}\phi_i + \frac{m_i}{2}\|e_{v_i}\|^2 + {\frac{3}{4k_{\alpha_i}}\widetilde{\alpha}_i^2} + \frac{1}{2k_{m_i}}\widetilde{m}_i^2 \bigg\}.
	\end{equation*} 
Since $x(t_0) \in \mathcal{F}$, $V(t_0)$ is bounded. 
Differentiation of $V$ yields, after using $\sum_{i\in\mathcal{N}}\sum_{j\in\mathcal{N}\backslash\{i\}}(x_i-x_j)^\top(v_i - v_j) = 2\sum_{i\in\mathcal{N}}\sum_{j\in\mathcal{N}\backslash\{i\}}(x_i-x_j)^\top v_i$, \eqref{eq:u_MAS}, \eqref{eq:adaptation laws}, and proceeding like in the proof of Theorem \ref{th:single robot}, yields
\begin{align*}
	\dot{V} =& \sum_{i\in\mathcal{N}}\Bigg\{ 2k_{\phi_i}k_{1_i}(x_i-x_{\textup{d}_i}) - 2k_{\phi_i}k_{2_i}\bigg(\beta_i'(d_{i,o_0})x - \sum_{k\in\mathcal{J}}\beta_i'(d_{i,o_k})(x_i - c_k)\\
  &\hspace{-5mm} - 2\sum_{j\in\mathcal{N}\backslash\{i\}}\beta'(d_{i,j})(x_i - x_j) - \sum_{j\in\mathcal{N}\backslash\{i\}} k_{f_i} \beta'_i(d_{i_{\textup{d}_j}})(x_i - x_{\textup{d}_j}) \bigg)^\top v_i + \\
	&\hspace{-6mm}e_{v_i}^\top(u_i - f_i(x_i,v_i) - m_ig - m_i\dot{v}_{\textup{d}_i} ) + \frac{1}{2k_{\alpha_i}}\widetilde{\alpha}_i\dot{\hat{\alpha}}_i + \frac{1}{2k_{m_i}}\widetilde{m}_i\dot{\hat{m}}_i \Bigg\}  \\
	\leq & \sum_{i\in\mathcal{N}}\bigg\{ k_{\phi_i} \nabla_{x_i}\widetilde{\phi}_i(x)^\top v_i + e_{v_i}^\top(u_i -m_i(g + \dot{v}_{\textup{d}_i}) ) + \\
	& \alpha_i\|e_{v_i}\|\|v_i\| + \widetilde{\alpha}_i\|e_{v_i}\|^2 - \widetilde{m}_ie_{v_i}^\top(\dot{v}_{\textup{d}_i} + g)  \bigg\},
\end{align*}
which, by using $v_i = e_{v_i} + v_{\textup{d}_i}$ and substituting the control and adaptation laws \eqref{eq:u_MAS},\eqref{eq:adaptation laws}, becomes
\begin{align*}
	\dot{V} \leq -\sum_{i\in\mathcal{N}}\bigg\{ \left(k_{\phi_i}-\frac{\alpha_i}{2}\right)\|\nabla_{x_i}\widetilde{\phi}_i(x)\|^2 + k_{v_i}\|e_{v_i}\|^2 \bigg\} \leq  0,
\end{align*}
and hence, $V(t) \leq V(t_0)$, which implies the boundedness of all {closed-loop} signals as well as that collisions between the agents and the agents and obstacles/workspace boundary are avoided $\forall t \geq t_0$. Moreover, following similar arguments as in the proof of Theorem \ref{th:single robot}, we conclude that $\lim_{t\to\infty}\|\nabla_{x_i}\widetilde{\phi}_i(x(t))\| = \lim_{t\to\infty}\|e_{v_i}(t)\| = \lim_{t\to\infty}\|v_i(t)\| = \lim_{t\to\infty}\|\dot{v}_i(t)\| =0$, $\forall i\in\mathcal{N}$. For the followers $\mathcal{N}_\mathcal{F}$, depending on the choice of $\tau_i$, $i\in\mathcal{N}_\mathcal{F}$, the critical point $\nabla_{x_i}\widetilde{\phi}_i(x(t)) = 0$ might either correspond to their destination $x_{\textup{d}_i}$ or a local minimum. In any case, it holds that $x(t)\in\mathcal{F}$, $\forall t \geq t_0$, and hence, for all the followers $i\in\mathcal{N}_\mathcal{F}$,
\begin{subequations} \label{eq:d_j static obs}
\begin{align}
	& \|x_i(t) - c_k\| > r_i + r_{o_k} + 2r_M + 2\bar{r},\forall k\in\mathcal{J} \label{eq:d_j static obs 1}\\
	&\|x_i(t) - x_j(t)\| > r_i + r_j + 2r_M + 2\bar{r}, 
	\forall j\in\mathcal{N}_\mathcal{F}\backslash\{i\} \label{eq:d_j static obs 2}\\
	& r_\mathcal{W} - \|x_i\|  > r_i + 2r_M + 2\bar{r}, \label{eq:d_j static obs 3}\\		
	& \|x_i(t) - x_{\textup{d}_j} \| > r_i + r_j + 2r_M + 2\bar{r} +\varepsilon, 
	 \forall {j\in\mathcal{N}_i} \label{eq:d_j static obs 4}
\end{align}
\end{subequations}
$\forall t > t_0$. Hence, since $\lim_{t\to\infty}\|v_i(t)\|$ $=$ $\lim_{t\to\infty}\|\dot{v}_i(t)\|$ $= 0$, $\forall i\in\mathcal{N}$, the multi-robot case reduces to the single-robot case of Section \ref{sec:main}, where the followers resemble static obstacles. Note that the obstacle constraints \eqref{eq:r_bar} are always satisfied by the followers (see \eqref{eq:d_j static obs 1}-\eqref{eq:d_j static obs 3}); \eqref{eq:d_j static obs 4} implies that the configuration that corresponds to the leader destination, i.e., $[x_1^\top$, $\dots$, $x_{i_\mathcal{L}-1}^\top$, $x_{\textup{d}_{i_\mathcal{L}}}^\top$, $x_{i_\mathcal{L}+1}^\top,\dots,x_N^\top]^\top$, belongs always in its free space $\mathcal{F}_{i_\mathcal{L}}$.
Hence, by choosing sufficiently small $\tau, \tau_{i_\mathcal{L}}$ in the interval $(0,\min(\bar{r}^2,\bar{r}_\textup{d}))$, with $\bar{r}_\textup{d}$ as defined in \eqref{eq:bar_r_d}, we guarantee the safe navigation of $x_{i_\mathcal{L}}$ to $x_{\textup{d}_{i_\mathcal{L}}}$ from almost all initial conditions, as in Section \ref{sec:main}.
\end{proof}
{When the current leader $i_\mathcal{L}$ reaches $\varepsilon$-close to its goal, at a time instant $t_{i_{\mathcal{L}}}$\footnote{Note that the proven asymptotic stability of Theorem \ref{th:MAS} guarantees that this will occur in finite time.}, it broadcasts this information to the other agents, switches off its control and remains immobilized, considered hence as a static obstacle with center $c_{M+1}\coloneqq x_{i_\mathcal{L}}(t_{i_{\mathcal{L}}})$ and radius $r_{M+1}$  by the rest of the team. Note that $\|c_{M+1} - x_{\textup{d}_{i_\mathcal{L}}}\| \leq \varepsilon$ and hence, in view of \eqref{eq:r_bar_mas}, $\|c_j - c_{M+1}\| > r_{o_j} + r_{i_\mathcal{L}} + 2r_M + 2\bar{r}$, $\forall j\in\mathcal{J}$, and $r_\mathcal{W} - \|c_{M+1}\| > r_{i_\mathcal{L}}+2r_M + 2\bar{r}$, satisfying the obstacle spacing properties \eqref{eq:r_bar}. 	
The next agent $i'_{\mathcal{L}}\in \widetilde{\mathcal{N}}\coloneqq \mathcal{N}\backslash \{i_{\mathcal{L}}\}$  in priority is then assigned as a leader for navigation, and we redefine the sets 
\begin{align*}
&\widetilde{\bar{\mathcal{O}}}_{i'_\mathcal{L},j} \coloneqq \{ q \in {\bar{\mathcal{W}}_{i'_\mathcal{L}}} : \|q_i-c_j\| \leq r_{o_j} + r_i \}, \forall j\in\widetilde{\mathcal{J}}\\
& \widetilde{\mathcal{C}}_{i'_\mathcal{L}} \coloneqq \{ q \in {\bar{\mathcal{W}}_{i'_\mathcal{L}}}: \|q_{i'_\mathcal{L}} - q_j \| \leq r_{i'_\mathcal{L}} + r_j, \forall j\in \widetilde{\mathcal{N}} \backslash \{i'_\mathcal{L}\}  \},  \\
&\widetilde{\bar{\mathcal{O}}}_{i,j} \coloneqq \{ q \in \bar{\mathcal{W}}_i : \|q_i-c_j\| \leq r_{o_j} + r_i + 2r_M + 2\bar{r} \}, \forall j\in\widetilde{\mathcal{J}} \\
&\widetilde{\mathcal{C}}_i \coloneqq \{ q \in \bar{\mathcal{W}}_i: \|q_i - q_{i'_\mathcal{L}} \| \leq r_i + r_{i'_\mathcal{L}},  \\
&\hspace{12mm}  \|q_i - q_j \| \leq r_i + r_j + 2r_M + 2\bar{r}, \forall j\in\widetilde{\mathcal{N}} \backslash \{i'_\mathcal{L}{,i}\}, \\
&\hspace{12mm}  \|q_i-x_{\textup{d}_j}\| \leq r_i + r_j + 2r_M + 2\bar{r} + \varepsilon, \forall j\in \widetilde{\mathcal{N}}_i  \},  
\end{align*}
$\forall i\in\widetilde{\mathcal{N}}\backslash\{i'_\mathcal{L}\}$, where $\widetilde{\mathcal{N}}_i\coloneqq \mathcal{N}_i\backslash\{i_{\mathcal{L}}\}$, and $\widetilde{\mathcal{J}} \coloneqq \mathcal{J}\cup\{M+1\}$, to account for the new obstacle $M+1$. The new free space is  $\widetilde{\mathcal{F}}_i \coloneqq \bar{\mathcal{W}}_i \backslash \{ ( \bigcup_{j\in\widetilde{\mathcal{J}}} \widetilde{\bar{\mathcal{O}}}_{i,j} ) \cup  \widetilde{\mathcal{C}}_i \}, \forall i\in\widetilde{\mathcal{N}}$ and, in view of \eqref{eq:d_j static obs}, one can conclude that $x_{i'_{\mathcal{L}}}(t_{i_{\mathcal{L}}}) \in \widetilde{\mathcal{F}}_{i'_{\mathcal{L}}}$, $x_i(t_{i_{\mathcal{L}}}) \in \widetilde{\mathcal{F}}_i$ $\forall i\in\widetilde{\mathcal{N}}\backslash\{i'_{\mathcal{L}}\}$. Therefore, the application of Theorem \ref{th:MAS} with $t_{i_{\mathcal{L}}}$ as $t_0$ and agent $i'_{\mathcal{L}}$ as leader guarantees its navigation $\varepsilon$-close to $x_{\textup{d}_{i'_{\mathcal{L}}}}$. Applying iteratively the aforementioned reasoning, we guarantee the successful navigation of all the agents {to their destinations}.}

\section{Simulation Results} \label{sec:sim}

This section verifies the theoretical findings of Sections \ref{sec:main}-\ref{sec:MAS} via computer simulations. 
We consider first a $2$D workspace on the horizontal plane with {$r_\mathcal{W} = 11$, populated with $M = 60$} randomly placed obstacles, whose radius, enlarged by the robot radius, is {randomly chosen in $\bar{r}_{o_j} \in [0.25,0.75]$}, $\forall j\in\mathcal{J}$. 
The mass, is chosen as $m=1$, and $f(x,v) = \frac{\alpha}{16} \sin(0.5(x_1 + x_2))F(v)v$, with $F(v) = \textup{diag}\{ [\exp(-\textup{sgn}(v_i)v_i) + 1]_{i\in\{1,2\}} \}$,
and $\alpha=10$, where we denote $(x_1,x_2) = x$, $(v_1,v_2) = v$. {Note that $f()$ is highly nonlinear, motivated by the friction model of \cite{de1995new}.}
We choose the goal position as $x_\textup{d} = (5,5)$, which the robot aims to converge to from $3$ different initial positions, namely $x(0) = -(5,5), (-7,3.5)$, and $(3.5,-7)$. The parameter $\bar{r}$ is chosen as $\bar{r} = 0.5$. We choose {a variation of \eqref{eq:beta_exps}} for $\beta$ with $\tau = \bar{r}^2$. The control gains are chosen as $k_1 = 0.04$, $k_2=5$, $k_v = 20$, $k_\phi = 1$, and $k_m = k_\alpha = 0.01$. The results for $t\in[0,100]$ seconds are depicted in Figs. \ref{fig:2d_traj_1_2_3}, \ref{fig:2d_u_1_2_3}; \ref{fig:2d_traj_1_2_3} (left) shows that the robot navigates to its destination without any collisions, and \ref{fig:2d_u_1_2_3} depicts the input and adaptation signals $u(t)$, $\hat{\alpha}(t)$, $\hat{m}(t)$ {for the trajectory starting from $(-5,5)$}. In addition, note that the fact that $\alpha > 2$ does not affect the performance of the proposed control protocol and hence we can verify that the condition $k_\phi > \frac{\alpha}{2}$ is only sufficient and not necessary. {Moreover, in order to verify the results of Section \ref{subsec:disturbance}, we add a bounded time-varying disturbance vector $d(x,v,t) = d(t) \coloneqq 2\left[\sin(0.5t+\frac{pi}{3}), \cos(0.4t-\frac{\pi}{4})\right]^\top \in \mathbb{R}^2$ and we choose the extra control gains as $\sigma_m = \sigma_\alpha = 0.1$. The results are depicted in Fig. \ref{fig:2d_traj_1_2_3} (right), which shows the safe navigation of the agent to a set close to $x_\textup{d}$, and Fig. \ref{fig:2d_u_1_2_3}, which shows the input and adaptation signals {for the trajectory starting from $(-5,5)$}.}

\begin{figure}[!ht]
	\centering
	\includegraphics[trim = 0cm 0cm 0cm -0cm,width = 0.85\textwidth]{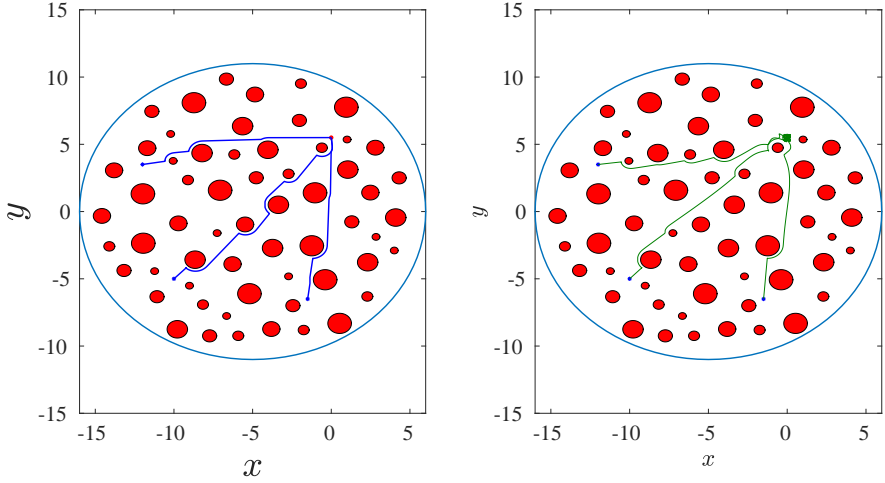}\\
	\caption{The resulting trajectories $x(t)$, $t\in[0,100]$ seconds, from the initial points $-(5,5), {(-7,3.5)}$, and $(3.5,-7)$ to the destination $(5,5)$. Left: without any disturbances. Right: with bounded disturbance $d(x,v,t)$. }\label{fig:2d_traj_1_2_3}
\end{figure}

\begin{figure}[!ht]
	\centering
	\includegraphics[trim = 0cm 0cm 0cm -0cm,width = 0.85\textwidth]{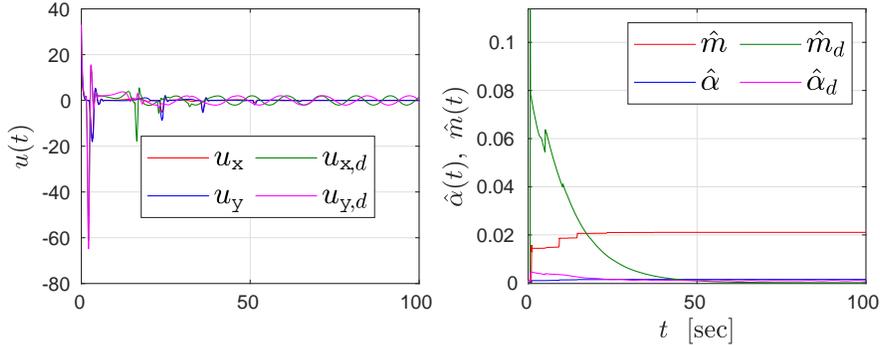}\\
	\caption{The input $u(t)=(u_1(t),u_2(t))$ (left), and adaptation signals $\hat{\alpha}(t)$, $\hat{m}(t)$ (right) for the $2$D trajectory {from $-(5,5)$ to $(5,5)$} of Fig. \ref{fig:2d_traj_1_2_3}. The subscript $d$ corresponds to the model where a bounded disturbance vector $d(x,v,t)$ was included. }\label{fig:2d_u_1_2_3}
\end{figure}

Next, we consider a $3$D workspace with {$r_\mathcal{W} = 11$, populated with $M = 200$ randomly placed obstacles, whose radius, enlarged by the robot radius, is randomly chosen in $\bar{r}_{o_j} \in [0.25,0.75]$}, $\forall j\in\mathcal{J}$; $f(x,v)$ and $m$ as well as the $\beta$ functions and control gains are chosen as in the $2$D scenario. 
We choose the goal position as $x_\textup{d} = (4,4,4)$, which the robot aims to converge to from $3$ different initial positions, namely $x(0) = -(4,4,4), (-4,4,-4)$, and $(-4,-4,4)$. The parameter $\bar{r}$ is chosen as $\bar{r} = 0.75$. The robot navigation as well as the input and adaptation signals $u(t)$, $\hat{\alpha}(t)$, $\hat{m}(t)$ {(for the trajectory starting from $-(4,4,4)$)} are depicted in Figs. \ref{fig:3d_traj_1_2_3}, and \ref{fig:3d_u_1_2_3} for $t\in[0,100]$ seconds. Note that the robot navigates to its destination without any collisions and that $\hat{m}$ converges to $m$, as predicted by the theoretical results.

\begin{figure}[!ht]
	\centering
	\includegraphics[trim = 0cm 0cm 0cm -0cm,width = 0.85\textwidth]{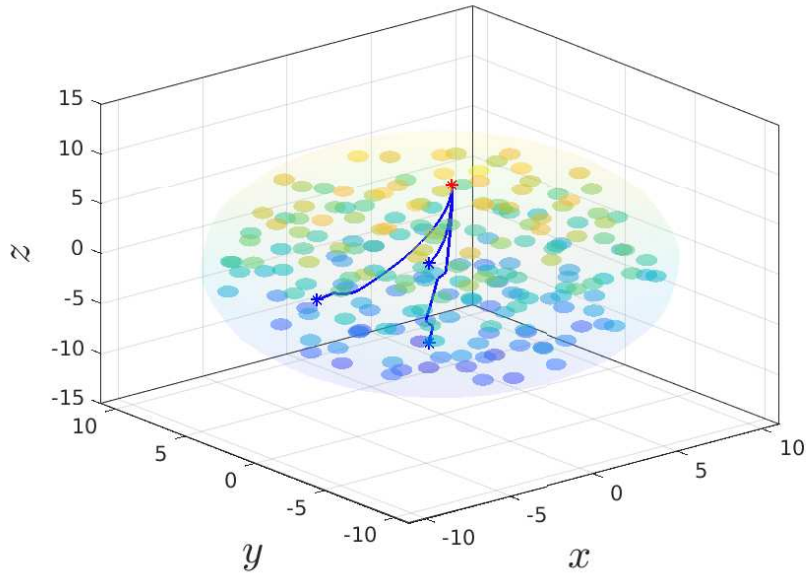}\\
	\caption{The resulting trajectories $x(t)$, $t\in[0,100]$ seconds, from the initial points $-(4,4,4), (-4,-4,4)$, and $(-4,4,-4)$ to the destination $(4,4,4)$.}\label{fig:3d_traj_1_2_3}
\end{figure}

\begin{figure}[!ht]
	\centering
	\includegraphics[trim = .5cm 0cm 0cm -0cm,width = 0.85\textwidth]{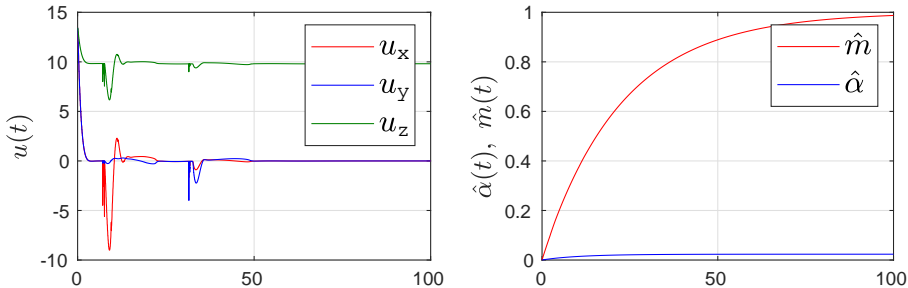}\\
	\caption{The input $u(t)=(u_1(t),u_2(t))$ (left), and adaptation signals $\hat{\alpha}(t)$, $\hat{m}(t)$ (right) for the $3$D trajectory from $-(4,4,4)$ to $(4,4,4)$ of Fig. \ref{fig:3d_traj_1_2_3}.}\label{fig:3d_u_1_2_3}
\end{figure}
%


{
Next, we illustrate the performance of the control protocol of Section \ref{sec:Star} in a $2$D star-world. We consider a 
workspace with $r_\mathcal{W} = 8$, which contains $2$ star-shaped obstacles, centered at $(-3,-3)$ and $(0,1)$, respectively. The mass $m$ and function $f(x,v)$ are {given} as before, with $\alpha = 1$.  
In order to transform the workspace to a sphere world, we employ the transformation proposed in \cite{rimon1992exact}. In the transformed sphere world, we choose  $\bar{r}=4$ and $\bar{r}_{o_j}=0.5$, whereas the function $\beta$ is chosen as in the sphere-world case. The initial and goal position are selected as $x(0) = (-5,-5)$ and $x_\textup{d} = (3,4)$, respectively, and the control gains as $k_1 = 0.04$, $k_2=.2$, $k_v = 20$, $k_\phi = 1$, and $k_m = k_\alpha = 0.01$. The robot trajectory is depicted in Fig. \ref{fig:2d_stars_traj}, for $t\in[0,500]$ seconds, both in the original star and in the transformed sphere world.}
%

{
Finally, we use the control scheme of Section \ref{sec:MAS} in a multi-agent scenario. We consider $20$ agents in a $2$D workspace of $r_\mathcal{W} = 120$, populated with $70$ obstacles. The agents and obstacles are randomly initialized to satisfy the conditions of the free space of Section \ref{sec:MAS}, {as shown in Fig. \ref{fig:mas_initial}}. The radius of the agents and the obstacles is chosen as $r_i=r_{o_j} = 2$, $\forall i\in\mathcal{N},j\in\mathcal{J}$, and the sensing radius of the agents is taken as $\varsigma_i = 20$, $\forall i\in\mathcal{N}$. The functions $\beta$, $\beta_i$ are chosen as before, and we also choose $\bar{r} = 4$, $\varepsilon = 0.1$. The results are depicted in Fig. \ref{fig:gamma+dis MAS} for $870$ seconds, which shows the convergence of the distance errors $\|x_i(t) - x_{\textup{d}-i}\|$ to zero, $\forall i\in\mathcal{N}$ as well as the minimum of the distances $\|x_i(t) - x_j(t)\| - 2r$, $\forall i,j\in\mathcal{N}$, $i\neq j$, and $\|x_i(t) - c_j\| - 2r$, $\forall i\in\mathcal{N}, j\in\mathcal{J}$, defined as $\beta_{\min}(t) \coloneqq \min \{\min_{i,j\in\mathcal{N},i\neq j}\{\|x_i(t)-x_j(t)\|-2r \}, 
\min_{(i,j)\in\mathcal{N}\times\mathcal{J}}\big\{ \|x_i(t)-c_j\|-2r \} \}$,
which stays strictly positive, $\forall t\in[0,870]$, implying that collisions are avoided. {A video illustrating the multi-agent navigation can be found in \href{https://vimeo.com/393443782}{https://vimeo.com/393443782}}.}

\begin{figure}[!ht]
	\centering
		\includegraphics[trim = 0cm 0cm 0cm -0cm,width = 0.85\textwidth]{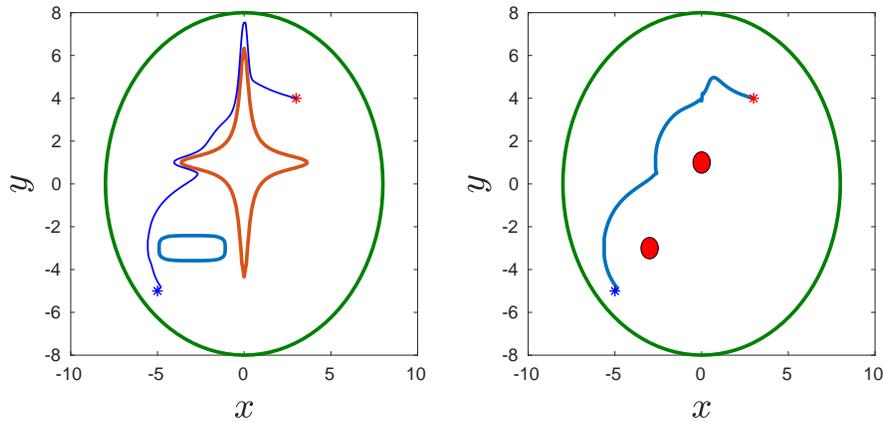}\\
	\caption{The resulting trajectory $x(t)$, $t\in[0,500]$ seconds, from the initial points $-(5,5)$ to the destination $(3,4)$, in the $2$D star world workspace (left) and the transformed sphere world (right).}\label{fig:2d_stars_traj}
\end{figure}


%


\begin{figure}[!ht]
	\centering
	\includegraphics[trim = 0cm 0cm 0cm -0cm,width = 0.85\textwidth]{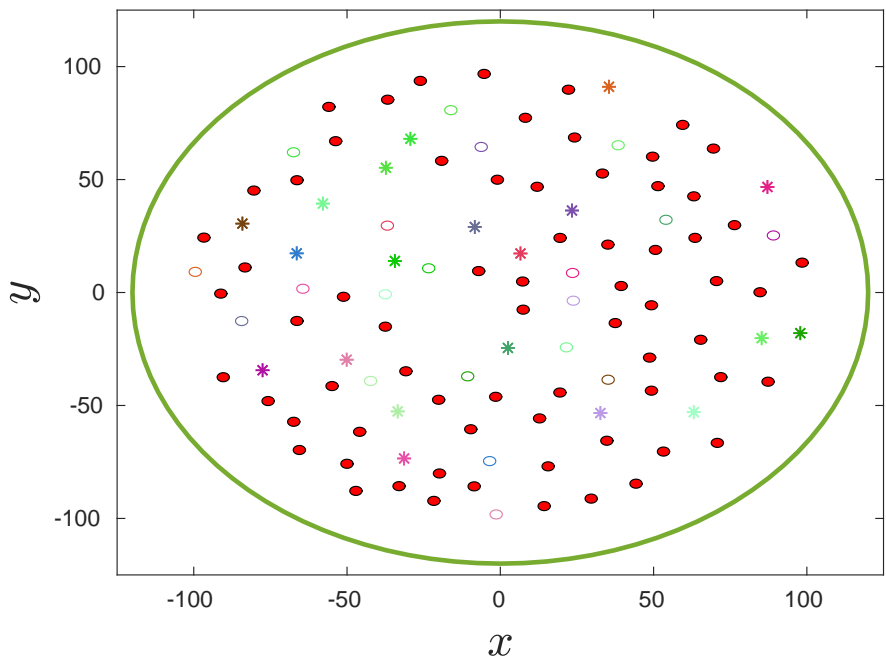}\\
	\caption{The initial configurations of the multi-agent scenario. The obstacles  are depicted as filled red disks whereas the agents as circles. The destinations are shown with asterisk.}\label{fig:mas_initial}
\end{figure}

\begin{figure}[!ht]
	\centering
	\includegraphics[trim = 0cm 0cm 0cm -0cm,width = 0.7\textwidth]{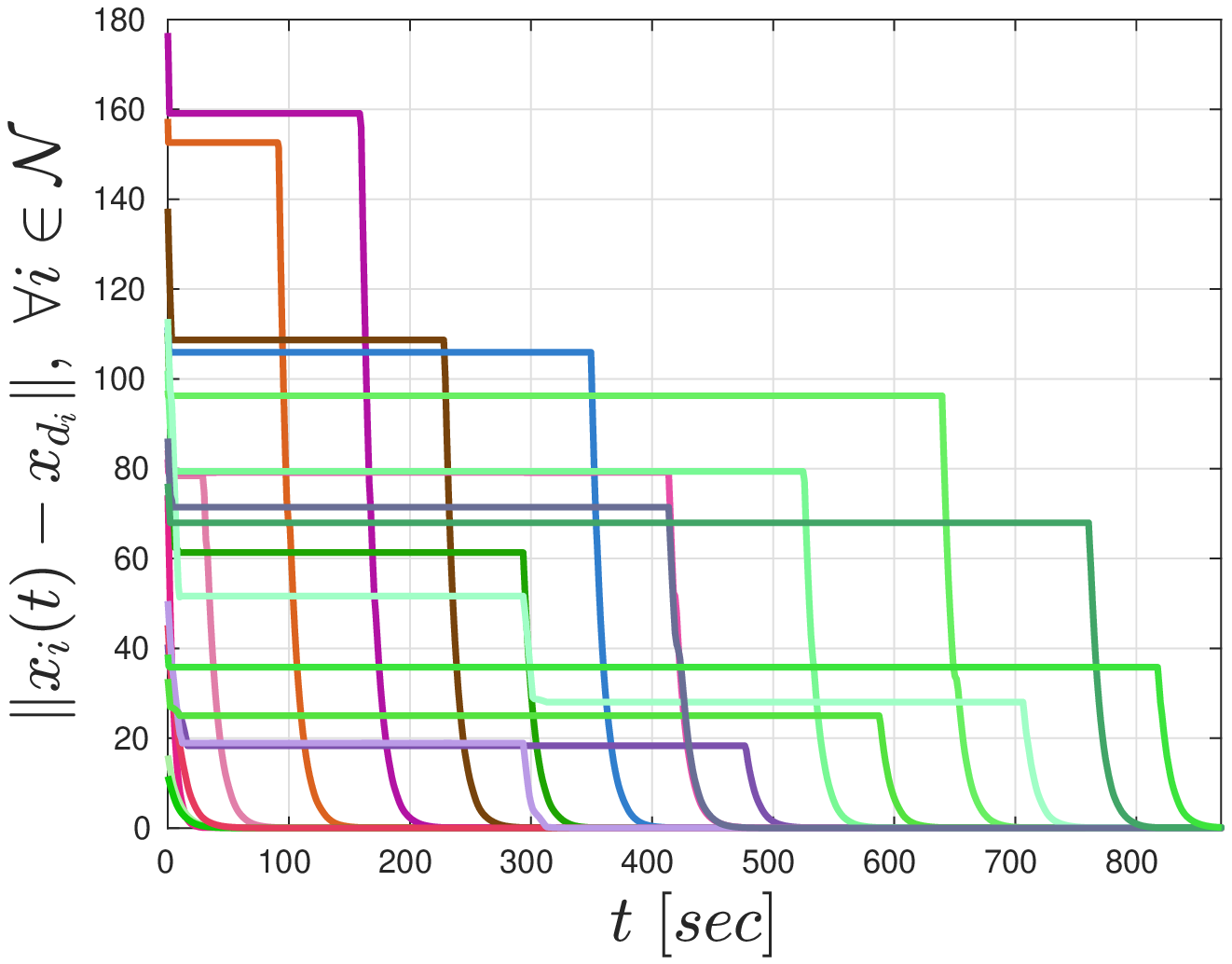}
	\includegraphics[trim = 0cm 0cm 0cm -0cm,width = 0.8\textwidth]{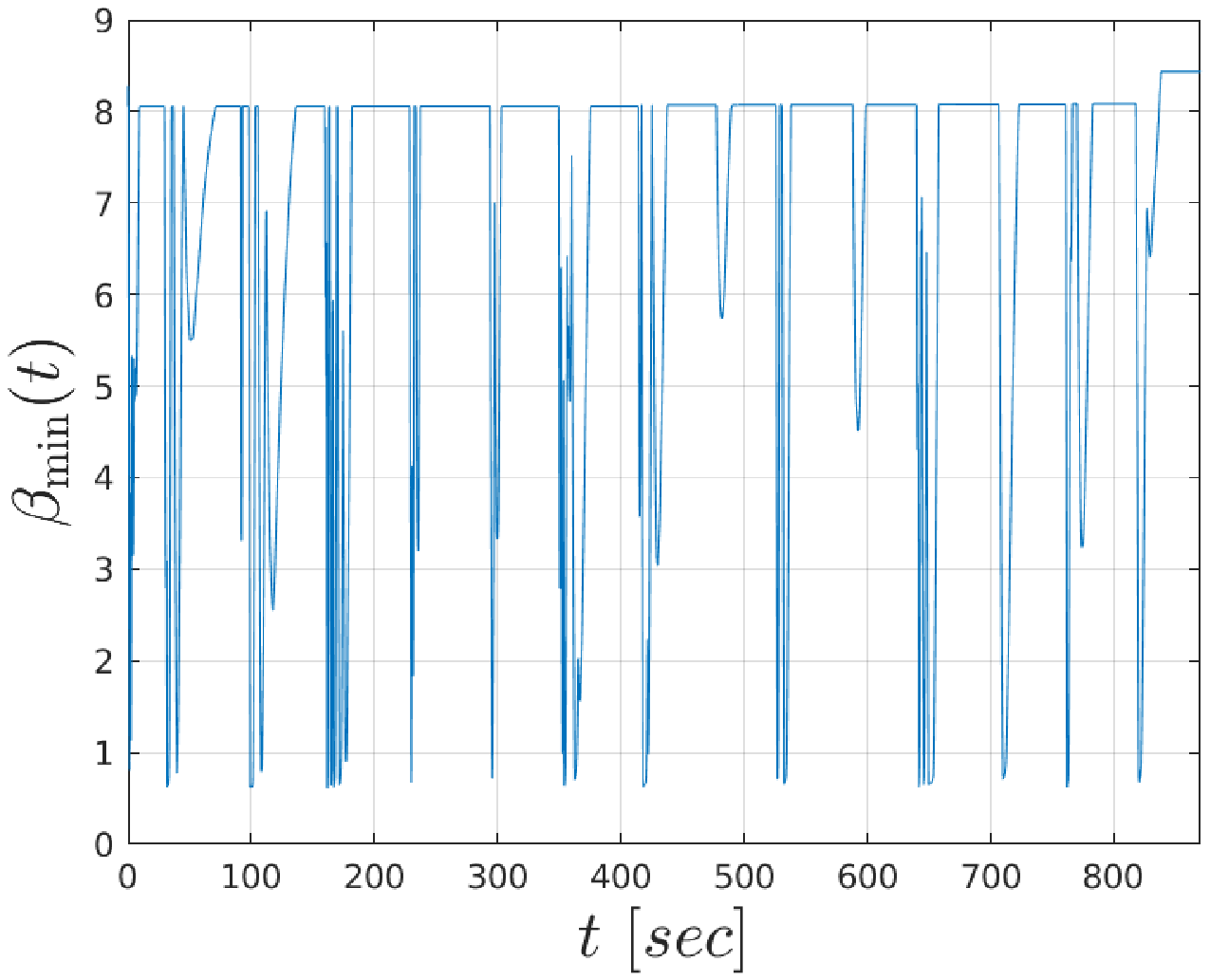}
	\caption{The resulting signals $\|x_i(t) - x_{\textup{d}_i}\|$, $\forall i\in\mathcal{N}$ (left) and the signal $\beta_{\min}(t)$ (right). }\label{fig:gamma+dis MAS}
\end{figure}

%


\section{Conclusion and Future Work} \label{sec:Concl}
{This paper considers the robot navigation in an obstacle-cluttered environment subject to uncertain $2$nd-order dynamics. A novel navigation function is proposed and combined with adaptation laws that compensate for the uncertain dynamics. 
The results are extended to star worlds as well as multi-agent cases. 
Future directions will aim at relaxing the assumptions for the latter.}
	
\bibliographystyle{unsrt}        
\bibliography{bibl}

\end{document}